\documentclass{article}
\usepackage{graphicx, fixltx2e, fullpage, setspace, amsmath, amsfonts, amsthm}
\usepackage{graphicx, color, import}
\usepackage{latexsym}

\newcommand{\e}{\varepsilon}

\begin{document}

\title{A Label Semantics Approach to Linguistic Hedges}
\date{28th January 2014}
\author{Martha Lewis and Jonathan Lawry \\
Department of Engineering Mathematics, \\University of Bristol, \\BS8 1TR, United Kingdom\\
\texttt{martha.lewis@bristol.ac.uk, j.lawry@bristol.ac.uk}}

\newtheorem{dfn}{Definition}
\newtheorem{thm}[dfn]{Theorem}
\newtheorem{cor}[dfn]{Corollary}
\newtheorem {exa}[dfn]{Example}

\maketitle

\begin{abstract}
We introduce a model for the linguistic hedges `very' and `quite' within the label semantics framework, and combined with the prototype and conceptual spaces theories of concepts. The proposed model emerges naturally from the representational framework we use and as such, has a clear semantic grounding. We give generalisations of these hedge models and show that they can be composed with themselves and with other functions, going on to examine their behaviour in the limit of composition.
\end{abstract}

\section{Introduction}
\label{sect:intro}

The modelling of natural language relies on the idea that languages are compositional, i.e. that the meaning of a sentence is a function of the meanings of the words in the sentence, as proposed by \cite{frege1948}. Whether or not this principle tells the whole story, it is certainly important as we undoubtedly manage to create and understand novel combinations of words. Fuzzy set theory has long been considered a useful framework for the modelling of natural language expressions, as it provides a functional calculus for concept combination \cite{zadeh1965, zadeh}. 

A simple example of compositionality is hedged concepts. Hedges are words such as `very', `quite', `more or less', `extremely'. They are usually modelled as transforming the membership function of a base concept to either narrow or broaden the extent of application of that concept. So, given a concept `short', the term `very short' applies to fewer objects than `short', and  `quite short' to more. Modelling a hedge as a transformation of a concept allows us to determine membership of an object in the hedged concept as a function of its membership in the base concept, rather than building the hedged concept from scratch \cite{zadehhedges}.

Linguistic hedges have been widely applied, including in fuzzy classifiers \cite{cetisli, chat, liu, marin2003} and database queries \cite{bordogna, boscempty}. Using linguistic hedges in these applications allows increased accuracy in rules or queries whilst maintaining human interpretability of results \cite{bouchon2009, marin2000}. This motivates the need for a semantically grounded account of linguistic hedges: if hedged results are more interpretable then the hedges used must themselves be meaningful.

In the following we provide an account of linguistic hedges that is both functional, and semantically grounded. In its most basic formulation, the operation requires no additional parameters, although we also show that the formulae can be generalised if  necessary. Our account of linguistic hedges uses the label semantics framework to model concepts \cite{lawry2004}. This is a random set approach which quantifies an agent's subjective uncertainty about the extent of application of a concept. We refer to this uncertainty as \emph{semantic uncertainty} \cite{lawry2009} to emphasise that it concerns the definition of concepts and categories, in contrast to stochastic uncertainty which concerns the state of the world. In \cite{lawry2009} the label semantics approach is combined with conceptual spaces \cite{gard2004} and prototype theory \cite{rosch}, to give a formalisation of concepts as based on a prototype and a threshold, located in a conceptual space. This approach is discussed in detail in section \ref{sec:bkg}. An outline of the paper is then as follows: section \ref{sec:lh} discusses different approaches to linguistic hedges from the literature, and compares these with our model. Subsequently, in section \ref{sec:results}, we give formulations of the hedges `very' and `quite'. These are formed by considering the dependence of the threshold of a hedged concept on the threshold of the original concept. We give a basic model and two generalisations, show that the models can be composed and investigate the behaviour in the limit of composition. Section \ref{sec:disc} compares  our results to those in the literature and proposes further lines of research.

\section{Theoretical approach to concepts}
\label{sec:bkg}
\subsection{Prototype theory and fuzzy set theory}
Prototype theory views concepts as being defined in terms of prototypes, rather than by a set of necessary and sufficient conditions. Elements from an underlying metric space then have graded membership in a concept depending on their similarity to a prototype for the concept. There is some evidence that humans use natural categories in this way, as shown in experiments  reported in \cite{rosch}. Fuzzy set theory \cite{zadeh1965} was proposed as a calculus for combining and modifying concepts with graded membership, and extended these ideas in \cite{zadeh} to linguistic variables as variables taking words as values, rather than numbers. For example, `height' can be viewed as a linguistic variable taking values `short,' `tall', `very tall', etc. The variable relates to an underlying universe of discourse $\Omega$, which for the concept `tall' could be $\mathbb{R}^+$. Then each value $L$ of the variable is associated with a fuzzy subset of $\Omega$, and a function $\mu_L:\Omega \rightarrow [0,1]$ associates  with each $x \in \Omega$ the value of its membership in $L$. Prototype theory gives a semantic basis to fuzzy sets through the notion of similarity to a prototype, as described in \cite{dpsemfuz}. In this context, concepts are represented by fuzzy sets and membership of an element in a concept is quantified by its similarity to the prototype. In this situation the fuzziness of the concept is seen as inherent to the concept. An alternative interpretation for fuzzy sets is random set theory, see \cite{dpsemfuz} for an exposition. Here, the fuzziness of a set comes from uncertainty about a crisp set, i.e. semantic uncertainty, rather than fuzziness inherent in the world. This second approach is the stance taken by \cite{lawry2009}, and which we now adopt in this paper.

\subsection{Conceptual Spaces}
\emph{Conceptual spaces} are proposed by G\"{a}rdenfors in \cite{gard2004} as a framework for representing information at the conceptual level. G\"{a}rdenfors contrasts his theory with both a symbolic, logical approach to concepts, and an associationist approach where concepts are represented as associations between different kinds of basic  information elements. Rather, conceptual spaces are geometrical structures based on quality dimensions such as weight, height, hue, brightness, etc. It is assumed that conceptual spaces are metric spaces, with an associated distance measure. This might be Euclidean distance, or any other appropriate metric. The distance measure can be used to formulate a measure of similarity, as needed for prototype theory - similar objects are close together in the conceptual space, very different objects are far apart.

To develop the conceptual space framework, G\"{a}rdenfors also introduces the notion of integral and separable dimensions. Dimensions are integral if assignment of a value in one dimension implies assignment of a value in another, such as depth and breadth. Conversely, separable dimensions are those where there is no such implication, such as height and sweetness. A \emph{domain} is then defined as a set of quality dimensions that are separable from all other dimensions, and a \emph{conceptual space} is defined as a collection of one or more domains.

G\"{a}rdenfors goes on to define a \emph{property} as a convex region of a domain in a conceptual space. A \emph{concept} is defined as a set of such regions that are related via a set of salience weights. This casting of (at least) properties as convex regions of a domain sits very well with prototype theory, as G\"{a}rdenfors points out. If properties are convex regions of a space, then it is possible to say that an object is more or less central to that region. Because the region is convex, its centroid will lie within the region, and this centroid can be seen as the prototype of the property. 

\subsection{Label Semantics}
\label{sub:lsem}
The label semantics framework was proposed by \cite{lawry2004} and related to prototype theory and conceptual spaces in \cite{lawry2009}. In this framework, agents use a set of labels $LA = \{L_1, L_2, ..., L_n\}$ to describe an underlying conceptual space $\Omega$ which has a distance metric $d(x,y)$ between points. In fact, it is sufficient that $d(x,y)$ be a pseudo-distance. When $x$ or $y$ is a set, say $Y$, we take $d(x,Y) = \text{min}\{d(x,y): y \in Y\}$. In this case, the set $Y$ is seen as an ontic set, i.e., a set where all elements are jointly prototypes, as opposed to an epistemic set describing a precise but unknown prototype, as described in \cite{dubois2012}. Each label $L_i$ is associated with firstly a set of prototype values $P_i \subseteq \Omega$, and secondly a threshold $\e_i$, about which the agents are uncertain. The thresholds $\e_i$ are drawn from probability distributions $\delta_{\e_i}$. Labels $L_i$ are associated with neighbourhoods $\mathcal{N}^{\e_i}_{L_i} = \{x \in \Omega : d(x, P_i) \leq \e_i\}$. The neighbourhood can be seen as the extension of the concept $L_i$. The intuition here is that $\e_i$ captures the idea of being sufficiently close to prototypes $P_i$. In other words, $x \in \Omega $ is sufficiently close to $P_i$ to be appropriately labelled as $L_i$ providing that $d(x,P_i) \leq \e_i$. 

Given an element $x \in \Omega$, we can ask how appropriate a given label is to describe it. This is quantified by an appropriateness measure, denoted $\mu_{L_i}(x)$. We are intentionally using the same notation as for the membership function of a fuzzy set. This quantity is the probability that the distance from $x$ to $P_i$, the prototype of $L_i$, is less than the threshold $\e_i$, as given by:

\[
\mu_{L_i}(x) = P(\e_i : x \in \mathcal{N}^{\e_i}_{L_i}) = P(\e_i : d(x, P_i) \leq \e_i) = \int_{d(x, P_i)}^\infty \delta_{\e_i}(\e_i) \mathrm{d}\e_i
\]

We also use the notation $\int_{d}^\infty \delta_{\e_i} (\e_i)\mathrm{d}\e_i = \Delta_i(d)$, according to which $\mu_{L_i}(x) = \Delta_i(d(x, P_i))$. The above formulation provides a link to the random set interpretation of fuzzy sets. Random sets are random variables taking sets as values. If we view $\mathcal{N}^{\e_i}_{L_i}$ as a random set from $\mathbb{R}^+$ into $2^\Omega$, then $\mu_{L_i}(x)$  is the single point coverage function of $\mathcal{N}^{\e_i}_{L_i}$, as defined in \cite{lawrybook}, and also commonly called a contour function \cite{shafer1976}.

Labels can often be semantically related to each other. For example, the label `pet fish' is semantically related to the labels `pet' and `fish', and the label `very tall' related to the label `tall'. This prompts two questions: firstly, how the prototypes of each concept are related to each other, and secondly, how the thresholds of each concept are related. Two simple models for the relationships between the thresholds are given in \cite{lawry2009}. The \emph{consonant model} takes all thresholds as being dependent on one common underlying threshold. So, all thresholds have the same distance metric $d$ and are related to a base threshold $\e$ by the dependency that $\e_i = f_i(\e)$ for increasing functions $f_i$. In contrast, the \emph{independence model} takes all thresholds as being independent of each other. This might hold when labels are taken from different conceptual spaces.

Between these two extremes, we model dependencies between thresholds as a Bayesian network - i.e., a directed acyclic graph whose edges encode conditional dependence between variables. The key property of this type of network is that the joint distribution of all variables can be broken into factors that depend only on each individual variable and its parents. So, for example, the network in figure \ref{fig:BN} can be factorised as $\delta(\e_1,\e_2, \e_3, \e_4, \e_5) = \delta_{\e_1}(\e_1)\delta_{\e_2}(\e_2)\delta_{\e_3|\e_1, \e_2}(\e_3|\e_1,\e_2)\delta_{\e_4|\e_2}(\e_4|\e_2)\delta_{\e_5|\e_3}(\e_5|\e_3)$.

\begin{figure}
  \centering
  \includegraphics[scale = 1]{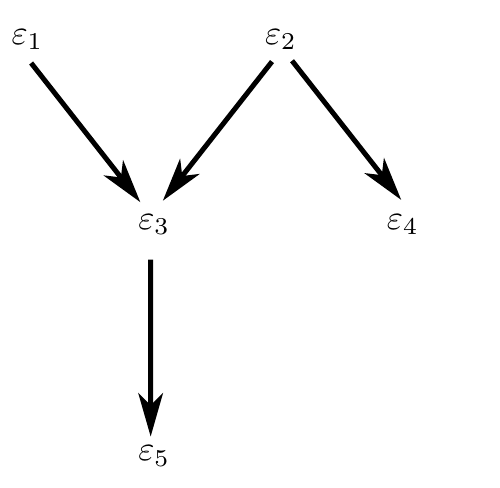}
  \caption{Example of a Bayesian network of thresholds. Dependencies between thresholds $\e_i$ are represented by arrows.}
  \label{fig:BN}
\end{figure}

This enables calculation of the joint distribution and therefore marginal distributions in an efficient manner.

One intuitively easy example is where the dependency of one threshold $\e_2$ on another $\e_1$ is that $\e_2 \leq \e_1$. This could be taken to model the dependency of the threshold of the concept `very tall' on the threshold of `tall'. The label `very tall' should be appropriate to describe fewer people than the label `tall'. Therefore, the  threshold for describing someone as `very tall' will be narrower than the threshold for describing someone as `tall', i.e. $\e_{\text{very tall}} \leq \e_{\text{tall}}$. This simple model will form part of the approach to modelling linguistic hedges, as outlined in the sequel.

\section{Approaches to linguistic hedges}
\label{sec:lh}
Linguistic hedges have been given varying treatments in the literature. In this section we summarise these different approaches and state the approach that we wish to take, discussing properties that hedge modifiers may need. We give two specific approaches from the literature with which we will compare our results.

In \cite{zadehhedges} the idea of linguistic hedges as operators modifying fuzzy sets was introduced, so that the membership function $\mu_{hL}(x)$ of a hedged concept, $hL$, is a function of the membership of the base concept $L$, i.e. $\mu_{hL}(x) = f(\mu_L(x))$. Furthermore, truth can be considered as a linguistic variable and hence a fuzzy set \cite{zadeh}, so that the application of a hedge can be seen as modifying the truth value of a sentence using that concept \cite{elsayed, honam, zadeh}. This second view is useful in approximate reasoning,  and allows for an algebraic approach to investigating the properties of linguistic hedges, as introduced in \cite{zadeh}, and expanded upon in \cite{howech, elsayed, honam}. The approach we take, however, is to view a hedge as modifying the fuzzy set associated with a concept directly, as taken by \cite{bosc, bouchon2009, dilascio1999, novak}. Rather than examining the algebraic properties of hedges or their role in reasoning, we look at how hedges are semantically grounded and argue that our approach provides a particularly clear semantics.

We will propose a set of operations that may be used for both expansion and refinement of single concepts. This is in contrast to the work presented in \cite{tang2009} in which information coarsening is effected by taking disjunctions of labels. The idea of a hedged concept has some similarities to that of the bipolar model of concepts described in \cite{tang2012}, since if it is appropriate to describe someone as `very tall', it must be appropriate to describe them as `tall', and similarly describing someone as `quite tall' implies that it is not entirely inappropriate to describe them as `tall'. However, we see the concepts derived by application of hedges as labels in their own right which can be used to describe data or objects. 

Zadeh divides hedges into two types. A type 1 hedge can be seen as an operator acting on a single fuzzy set. Examples are `very', `more or less', `quite', or `extremely' \cite{zadehhedges}. Type 2 hedges are more complicated and include modifiers such as `technically' or `practically'. In \cite{zadehhedges} concepts are considered as made up of various different components, with the membership function a weighted sum of the memberships of the individual components. Type 1 hedges operate on all components equally, whereas type 2 hedges differentiate between components. For example, the hedge `essentially' might give more weight to the most important components in a concept. Type 2 hedges are further explored in \cite{lakoffhedges, waart}, where components of a concept are categorised as \emph{definitional}, \emph{primary} or \emph{secondary}, and the hedges `technically', `strictly speaking' and `loosely speaking' are analysed in terms of these categories. Although in the following we restrict ourselves to consideration of type 1 hedges only, the treatment of concepts as having different components is mirrored by the conceptual spaces view, where each component might be seen as a dimension in the conceptual space. Further development of the framework may therefore allow a treatment of type 2 hedges.

A further distinction between types of hedge lies in the difference between powering or shifting modifiers. Powering modifiers are of the form $\mu_{hL}(x) = (\mu_L(x))^k$, where $hL$ refers to the hedged concept and $k$ is some real value, and shifting modifiers are of the form $\mu_{hL}(x) = (\mu_L(x-a))$. Zadeh introduces both types of modifier in his discussion of type 1 hedges \cite{zadehhedges}, however his powering modifiers are most frequently cited. These are the \emph{concentration} operator $CON(\mu_{\text{tall}}(x)) = (\mu_{\text{tall}}(x))^2$, and the \emph{dilation} operator $DIL(\mu_{\text{tall}}(x)) = (\mu_{\text{tall}}(x))^\frac{1}{2}$, which are often taken to implement the hedges `very' and `quite', (alternatively `more or less'), respectively. \footnote{Zadeh in fact proposes a rather more complicated hedge for `more or less', which involves a combination of powering and shifting, however, the dilation operator is more frequently quoted in the literature.}

The operators $CON$ and $DIL$ leave the core, $\{x \in \Omega : \mu_L(x) = 1\}$, and support  $\{x \in \Omega : \mu_L(x) \neq 0\}$, of the fuzzy sets unchanged, which is often argued to be undesirable \cite{boscempty, bosc, novak, marin}. In particular, \cite{boscempty} argue that in a fuzzy database, if a concentrating hedge is being used to refine a query that is returning too many objects, the hedge needs to reduce the number of objects returned, and hence narrow down the core. Furthermore, \cite{marin2003} find that classifiers using the $CON$ and $DIL$ operators (classical hedges) do not perform as well as those with hedges that modify the core and support of the fuzzy sets. In contrast, Zadeh himself argues that the core should not be altered. The application of a modifier `very' to a property given by a crisp set should leave that property unchanged: `very square' is the same as `square'. A fuzzy set is made up of a non-fuzzy part, the core, and a fuzzy part, $\{x \in \Omega : 0 < \mu_L(x) <1\}$. Since the core of a fuzzy set is a crisp set, it should be left unchanged. The use of classical hedges does improve performance over non-hedged fuzzy rules in expert systems \cite{cetisli, chat, liu}, so the argument against classical hedges is a matter of degree.

The use of the $CON$ and $DIL$ operators to model the hedges `very' and `quite' is further criticised on the basis that the modifiers are arbitrary and semantically ungrounded. No justification is given for these modifiers other than that they have what seem to be intuitively the right properties \cite{bosc, decock2002, novak}. Grounding hedges semantically is important for a theoretical account of what happens when we use terms like `very' and also for retaining interpretability in fuzzy systems. \cite{bosc, decock2002} both ground modifiers using a resemblance relation which takes into account how objects in the universe are similar to each other. \cite{novak} takes a horizon shifting approach.

In \cite{novak} the class of finite numbers is used as an example of the horizon shifting approach. Some numbers are certainly finite, however as numbers get larger, finiteness becomes impossible to verify. Mapping this idea onto the concept `small', we can say that there is a class of numbers that are definitely small, say $[0, c]$. As numbers get larger than $c$ we approach the horizon past which the concept `small' no longer applies, expressed as $1 - \epsilon(x)(x-c)$. So:

\[
\mu_{\text{small}}(x) = 
\begin{cases}
1 & \text{if } x \in [0, c] \\
1 - \epsilon(x)(x-c) & \text{if } x \geq c
\end{cases}
\]

Now, to implement the hedge `very', the horizon $c$ is shifted by a factor $\sigma$ and the membership function altered thus:

\[
\mu_{\text{very small}}(x) = 
\begin{cases}
1 & \text{if } x \in [0, \sigma c] \\
1 - \epsilon(x, \sigma)(x- \sigma c) & \text{if } x \geq  \sigma c
\end{cases}
\]

In \cite{novak}, examples of different kinds of membership functions that might be used to implement this idea are given. A linear membership function gives $\epsilon(x) = \frac{1}{a-c}$ where $a$ is the upper limit of the membership function. To implement the hedge, the function $\epsilon(x, \sigma)  = \frac{1}{\sigma (a-c)}$ is introduced, giving
\label{eqn:novak}
\[
\mu_{\text{small}}(x) = 
\begin{cases}
1 & \text{if } x \in [0, c] \\
1 - \frac{x-c}{a-c} & \text{if } x \in [c, a]\\
0 & \text{otherwise}
\end{cases}
\]
and 
\[
\mu_{\text{very small}}(x) = 
\begin{cases}
1 & \text{if } x \in [0, \sigma c] \\
1 - \frac{x- \sigma c}{\sigma(a-c)} & \text{if } x  \in [\sigma c, \sigma a]\\
0 & \text{otherwise}
\end{cases}
\]

\cite{bosc, decock2002} both ground their approaches in the idea of looking at the elements near a fuzzy set in order to contract or dilate the set. The two approaches are similar, so we restrict ourselves to that of \cite{bosc}. This approach introduces a fuzzy resemblance relation on the universe of discourse, and either a $T$-norm in the case of dilation, or a fuzzy implicator for concentration. The modifier is then implemented as follows. Consider a fuzzy set $F$ and a proximity relation $E^Z$ which is approximate equality, parametrised by a fuzzy set $Z$. As described in \cite{bosc}, $E$ is modelled by $(u,v) \rightarrow E(u,v) = Z(u - v)$, where $Z$ is a fuzzy interval centred on $0$ with finite support. In terms of a trapezoidal membership function, $Z$ can be expressed as $(-z - a, -z, z, z+a)$. Therefore, if $|u -v| \leq z$, $u$ and $v$ are judged to be approximately equal, i.e. $E^Z(u, v) = 1$. The set $F$ is dilated by $E^Z(F)(s) = \text{sup}_{r \in \Omega} T(F(r), E^Z(s,r))$, where $T$ is any $T$-norm, $min$ being the standard. 

To understand the effect that this has on a fuzzy set $F$, suppose that $F$ has a trapezoidal membership function $(A, B, C, D)$ where $[B, C]$ is the core of $F$ and $[A,B]$, $[C,D]$ the support, and that $Z$ similarly is $(-z - a, -z, z, z+a)$, with the T-norm $min$ used. Then $E^Z(F) = (A - z - a, B - z, C +z, D +z +a)$.

 Concentration is effected in a similar way:  $E_Z(F)(s) = \text{inf}_{r \in \Omega} I(F(r), E^Z(s,r))$, where I is a fuzzy implication.  If $F$ and $Z$ are as above with the condition that $C - B \geq 2z$, and $I$ is the G\"{o}del implication, then $E_Z(F)$ = $(A + z + a, B + z, C - z, D - z - a)$. 

For example, suppose we start with a set $F$ described in trapezoidal notation as $F = (A, B, C, D) = (2,4,6,8)$, and an approximate equality function parametrised by $Z =(-z-a, -z, z, z + a) = (-1, -0.5, 0.5, 1)$. The dilation of the set $F$ using T-norm $min$ is then:

\[
E^Z(F) = (A - z - a, B - z, C +z, D +z +a) = (1, 3.5, 6.5, 9)
\]

The concentration of the set $F$ using the G\"{o}del implication is:

\[
E_Z(F) = (A + z + a, B + z, C - z, D - z - a) = (3, 4.5, 5.5, 7)
\]

These effects are illustrated in figure \ref{fig:bosc_plot}.

\begin{figure}[htbp]
\centering
\includegraphics[scale = 0.5]{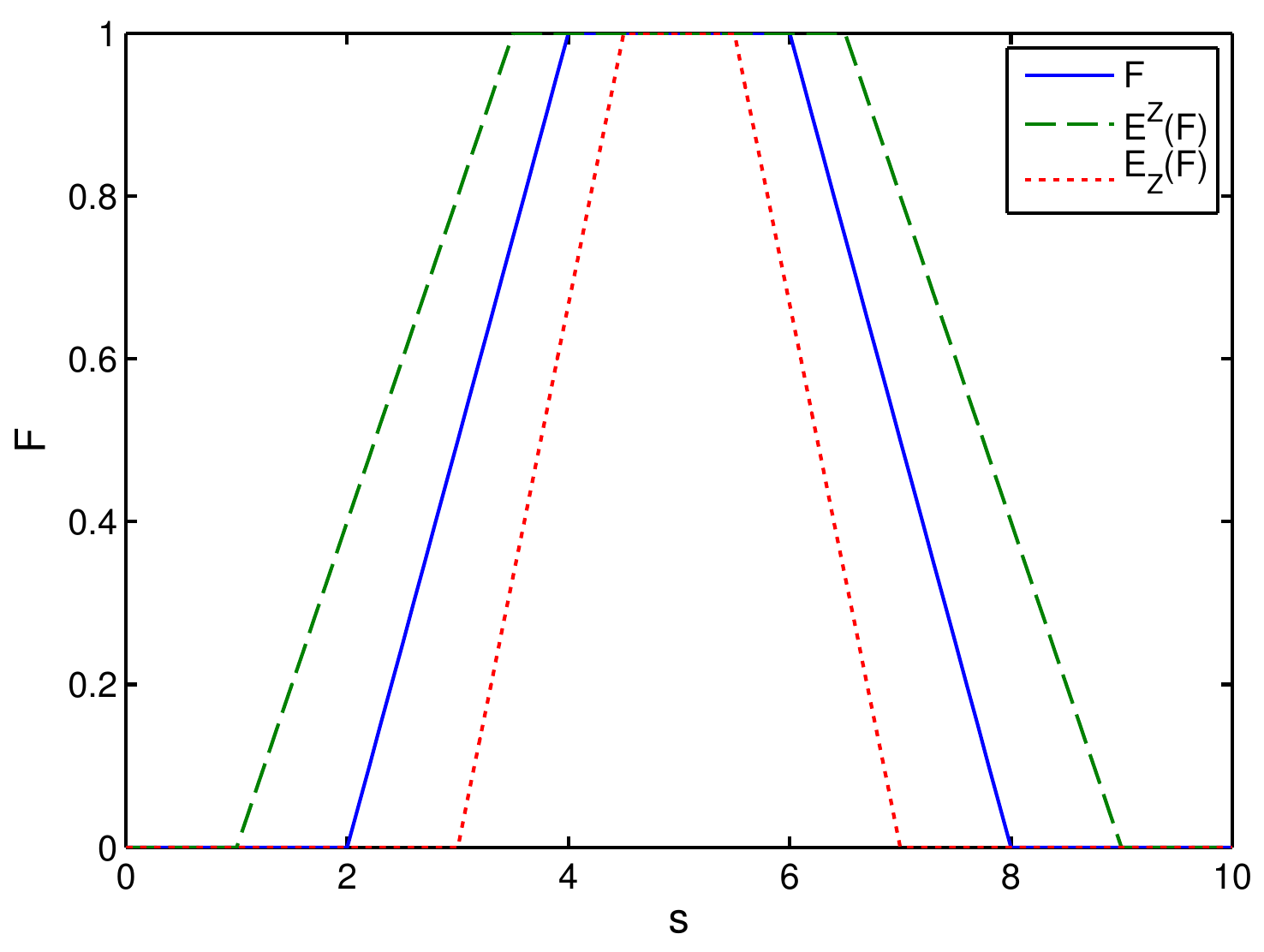}
\caption{Illustration of the expansion and contraction modifiers proposed in \cite{bosc}. The original set $F$ can be described in trapezoidal set notation as $(2,4,6,8)$. The approximate equality function is parametrised by a set $Z = (-1, -0.5,0.5, 1)$. Dilating the set as described in \cite{bosc} gives the set $E^Z(F) = (1,3.5,6.5,9)$. Concentrating the set as described  results in the set $E_Z(F) = (3,4.5, 5.5, 7)$ }
\label{fig:bosc_plot}
\end{figure}

The intuitive idea behind this approach is that if an object $x_1$ resembles another object $x_2$ that is $L$, then $x_1$ can be said to be `quite $L$'. Conversely, object $x_2$ that is $L$ can be said to be `very $L$' only if all the objects $x$ that resemble it can be said to be $L$. This formulation alters both the core and support of the fuzzy set $L$, which has been argued to be a desirable effect.

Following \cite{bosc, decock2002, novak}, we will propose linguistic modifiers that are semantically grounded rather than attempting to show their utility in classifiers, reasoning or to examine the algebra of modifiers. Our approach to linguistic modifiers arises very naturally from the label semantics framework, and the primary result does not require any parameters additional to the original membership function of the concept. We also show similarities between our model and the two detailed above.

\section{Label semantics approach to linguistic hedges}
\label{sec:results}
We present three formulations of linguistic hedges with increasing levels of generality. The first assumes that prototypes are equal. Secondly, we show that an analogue holds where prototypes are not equal, and thirdly that these hold in the case where the second threshold is a function of the first. We go on to show similarities between our model and those of \cite{bosc, decock2002, novak}. Furthermore, we show that hedges are compositional, and look at their behaviour in the limit of composition.

As described in section \ref{sub:lsem}, $LA$ denotes a finite set of labels $\{L_i\}$ that agents use to describe basic categories. $\Omega$ is the underlying domain of discourse, with prototypes $P_i \in \Omega$ and thresholds $\e_i$, drawn from a distribution $\delta_{\e_i}$. As before, the appropriateness $\mu_{L_i}(x) = \Delta_i(d(x, P_i)) = \int_{d(x, P_i)}^\infty \delta_{\e_i}(\e_i)\mathrm{d}\e_i$. We use the notation $L_i = <P_i, d, \delta_{\e_i}>$.

A concept $L_1$ can be narrowed or broadened to a second concept $L_2$ using the linguistic hedges `very' and `quite' respectively, i.e. $L_2$ is defined as `quite $L_1$'. The directed acyclic graph illustrating this dependency is given in figure \ref{fig:quite}. In this case, the threshold $\e_2$ associated with $L_2$ is dependent on $\e_1$ in that $\e_2 \geq \e_1$. In the case of `very', we have that $\e_2 \leq \e_1$.  Essentially, for `quite', we are saying that however wide a margin of certainty we apply the label `tall' with, the margin for `quite tall' will be wider, and conversely for `very'. 

\begin{figure}
  \centering
  \includegraphics[scale = 1]{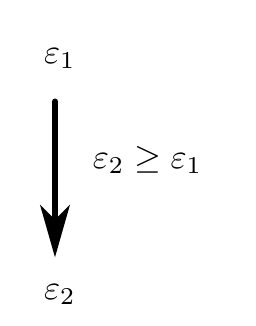}
	  \caption{Directed acyclic graph representing the hedge `quite'. The threshold $\e_2$ is dependent on the threshold $\e_1$ by $\e_2 \geq \e_1$.}
	\label{fig:quite}
\end{figure}

\subsection{Hedges with unmodified prototypes}
\label{sub:basic}
\begin{dfn}[Dilation and Concentration]
A label $L_2 = <P_2, d, \delta_{\e_2}>$  is a \emph{dilation} of a label $L_1 = <P_1, d, \delta_{\e_1}>$ when $\e_2$ is dependent on $\e_1$ such that $\e_2 \geq \e_1$. $L_2$ is a \emph{concentration} of $L_1$ when $\e_2$ is dependent on $\e_1$ such that $\e_2 \leq \e_1$.
\end{dfn}

\begin{thm}[$L_2 = $ quite $L_1$]
\label{thm:dil}
Suppose $L_2 = <P_2, d, \delta_{\e_2}>$ is a dilation of $L_1 = <P_1, d, \delta_{\e_1}>$, so that $\e_2 \geq \e_1$. Suppose also that $P_1 = P_2 = P$, and that the marginal (unconditional) distribution of $\e_2$, before conditioning on the knowledge that $\e_2 \geq \e_1$, is identical to $\delta_{\e_1}$, since $L_2$ is a dilation of $L_1$. Then $\forall x \in \Omega$, $\mu_{L_2}(x) = \mu_{L_1}(x) - \mu_{L_1}(x)\ln(\mu_{L_1}(x))$.
\end{thm}
\begin{proof}
\begin{align*}
	\delta_{\e_2|\e_1}(\e_2|\e_1) &= 
	\begin{cases}
		\frac{\delta_{\e_1}(\e_2)}{\int_{\e_1}^\infty \delta_{\e_1}(\e_2) d\e_2} & \text{if } \e_2 \geq \e_1\\
		0 & \text{otherwise}
	\end{cases} \quad =
	\begin{cases}
		\frac{\delta_{\e_1}(\e_2)}{\Delta_1(\e_1)} & \text{if } \e_2 \geq \e_1\\
		0 & \text{otherwise}
	\end{cases}
\end{align*}

and hence,

\begin{align*}
	\delta(\e_1, \e_2) &= \delta_{\e_1}(\e_1)\delta_{\e_2|\e_1}(\e_2|\e_1)
	= 
	\begin{cases}
		\frac{\delta_{\e_1}(\e_1)\delta_{\e_1}(\e_2)}{\Delta_1(\e_1)} & \text{if } \e_2 \geq \e_1\\
		0 & \text{otherwise}
	\end{cases}
\end{align*}

Then since $\e_2 \geq \e_1$ we have that

\begin{align*} 
	\mu_{L_2}(x) &= \int_0^\infty\int_{max(\e_1, d(x,P))}^\infty \delta(\e_1, \e_2) d\e_2 d\e_1 = \int_0^\infty\int_{max(\e_1, d(x,P))}^\infty \frac{\delta_{\e_1}(\e_1)\delta_{\e_1}(\e_2)}{\Delta_1(\e_1)} d\e_2 d\e_1\\
	&= \int_0^{d(x,P)}\frac{\delta_{\e_1}(\e_1)}{\Delta_1(\e_1)}\int_{d(x,P)}^\infty \delta_{\e_1}(\e_2) d\e_2 d\e_1 + \int_{d(x,P)}^\infty\frac{\delta_{\e_1}(\e_1)}{\Delta_1(\e_1)}\int_{\e_1}^\infty \delta_{\e_1}(\e_2) d\e_2 d\e_1\\
	&= \mu_{L_1}(x)\int_0^{d(x,P)}\frac{\delta_{\e_1}(\e_1)}{\Delta(\e_1)} d\e_1 + \int_{d(x,P)}^\infty \delta_{\e_1}(\e_1) d\e_1 = \mu_{L_1}(x) - \mu_{L_1}(x)\ln(\mu_{L_1}(x))\nonumber
\end{align*}
\end{proof}
The following example gives an illustration of the effect of applying this hedge, in comparison with the standard dilation hedge $(\mu_L(x))^{1/2}$.
\begin{exa}
\label{exa:dil_plot}
Suppose our conceptual space $\Omega = \mathbb{R}$ with Euclidean distance and that a label $L$ has prototype $P = 5$,  and threshold $\e \sim $ Uniform$[0,3]$. Then 
\[
\mu_L(x)  = 
\begin{cases}
1 - \frac{|x - 5|}{3} & \text{ if } x \in [2, 8]\\
0 & otherwise
\end{cases}
\]

We can then form a new label $qL$ with prototype $P_q = P = 5$ and threshold $\e_q \geq \e$. Then, according to theorem \ref{thm:dil}, $\mu_{qL}(x) = \mu_L(x) - \mu_L(x)\ln \mu_L(x)$. The effect of applying a dilation hedge to $L$ can be seen in figure \ref{fig:dil_plot}. The dilation hedge given above is contrasted with Zadeh's dilation hedge $(\mu_L(x))^{1/2}$.

\begin{figure}[htbp]
\centering
\includegraphics[scale = 0.5]{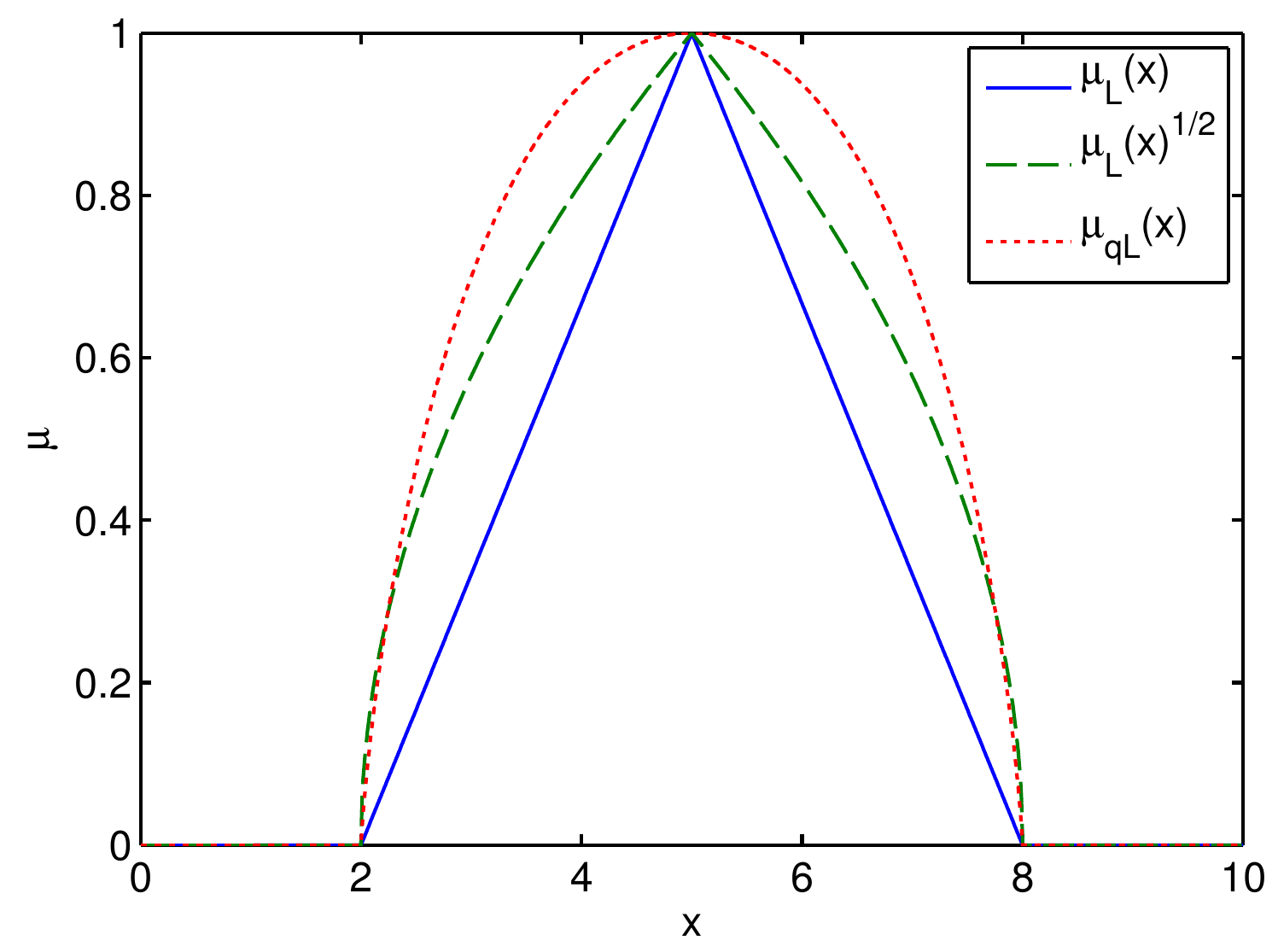}
\caption{Plots of $\mu_L(x)$, $\mu_L(x)^{1/2}$, and $\mu_{qL}(x) = \mu_L(x) - \mu_L(x)\ln \mu_L(x)$. Values of $\mu_{qL}(x)$ near to the prototype $P = 5$ are greater than equivalent values of $\mu_L(x)^{1/2}$.}
\label{fig:dil_plot}
\end{figure}
\end{exa}

\begin{thm}[$L_2 = $ very $L_1$]
\label{thm:con}
Suppose $L_2 = <P_2, d, \delta_{\e_2}>$ is a concentration of $L_1 = <P_1, d, \delta_{\e_1}>$, so that $\e_2 \leq \e_1$. Suppose also that $P_1 = P_2 = P$, and that the marginal (unconditional) distribution of $\e_2$, before conditioning on the knowledge that $\e_2 \leq \e_1$, is identical to $\delta_{\e_1}$, since $L_2$ is a concentration of $L_1$. Then $\forall x \in \Omega$, $\mu_{L_2} = \mu_{L_1}(x)  +  (1-\mu_{L_1}(x))\ln(1 - \mu_{L_1}(x))$.
\end{thm}
\begin{proof}
\begin{align*}
	\delta_{\e_2|\e_1}(\e_2|\e_1) &= 
	\begin{cases} 
		\frac{\delta_{\e_1}(\e_2)}{\int_0^{\e_1} \delta_{\e_1}(\e_2) d\e_2} & \text{if } \e_2 \leq \e_1\\
		0 & \text{otherwise}
	\end{cases} =
	\begin{cases}
		\frac{\delta_{\e_1}(\e_2)}{1-\Delta_1(\e_1)} & \text{if } \e_2 \leq \e_1\\
		0 & \text{otherwise}
	\end{cases}
\end{align*}

and hence,

\begin{align*}
	\delta(\e_1, \e_2) &= \delta_{\e_1}(\e_1)\delta_{\e_2|\e_1}(\e_2|\e_1)
	= 
	\begin{cases}
		\frac{\delta_{\e_1}(\e_1)\delta_{\e_1}(\e_2)}{1-\Delta_1(\e_1)} & \text{if } \e_2 \leq \e_1\\
		0 & \text{otherwise}
	\end{cases}
\end{align*}

So since $\e_2 \leq \e_1$ we have that:

\begin{align*}
	 \mu_{L_2}(x) &= \int_0^\infty \int_{min(\e_1, d(x,P))}^{\e_1} \delta(\e_1, \e_2) d\e_2 d\e_1 = \int_0^\infty\int_{min(\e_1, d(x,P))}^{\e_1} \frac{\delta_{\e_1}(\e_1)\delta_{\e_1}(\e_2)}{1-\Delta_1(\e_1)} d\e_2 d\e_1\\
	&= \int_{d(x,P)}^\infty \frac{\delta_{\e_1}(\e_1)}{1 - \Delta_1(\e_1)} \int_{d(x,P)}^{\e_1} \delta_{\e_1}(\e_2) d\e_2 d\e_1\\
	&= \int_{d(x,P)}^\infty \frac{\delta_{\e_1}(\e_1)}{1 - \Delta_1(\e_1)} \left( \int_0^{\e_1} \delta_{\e_1}(\e_2) d\e_2 -  \int_0^{d(x,P)} \delta_{\e_1}(\e_2) d\e_2 \right) d\e_1\\ 
	&= \mu_{L_1}(x)  -  (1-\mu_{L_1}(x))\int_{d(x,P)}^\infty \frac{\delta_{\e_1}(\e_1)}{1 - \Delta_1(\e_1)} d\e_1 = \mu_{L_1}(x)  +  (1-\mu_{L_1}(x))\ln(1 - \mu_{L_1}(x))
\end{align*}

\end{proof}

The effect of these hedges are illustrated in the following example.

\begin{exa}
Suppose the label $L$ is as described in example \ref{exa:dil_plot}. We can form a new label $vL$ with prototype $P_v = P = 5$ and threshold $\e_v \leq \e$. Then, according to theorem \ref{thm:con}, $\mu_{vL}(x) = \mu_L(x) + (1 - \mu_L(x))\ln(1- \mu_L(x))$  The effect of applying this contraction hedge is seen in figure \ref{fig:con_plot}, and again, this concentration hedge is contrasted with Zadeh's concentration hedge $(\mu_L(x))^2$.

\begin{figure}[htbp]
\centering
\includegraphics[scale = 0.5]{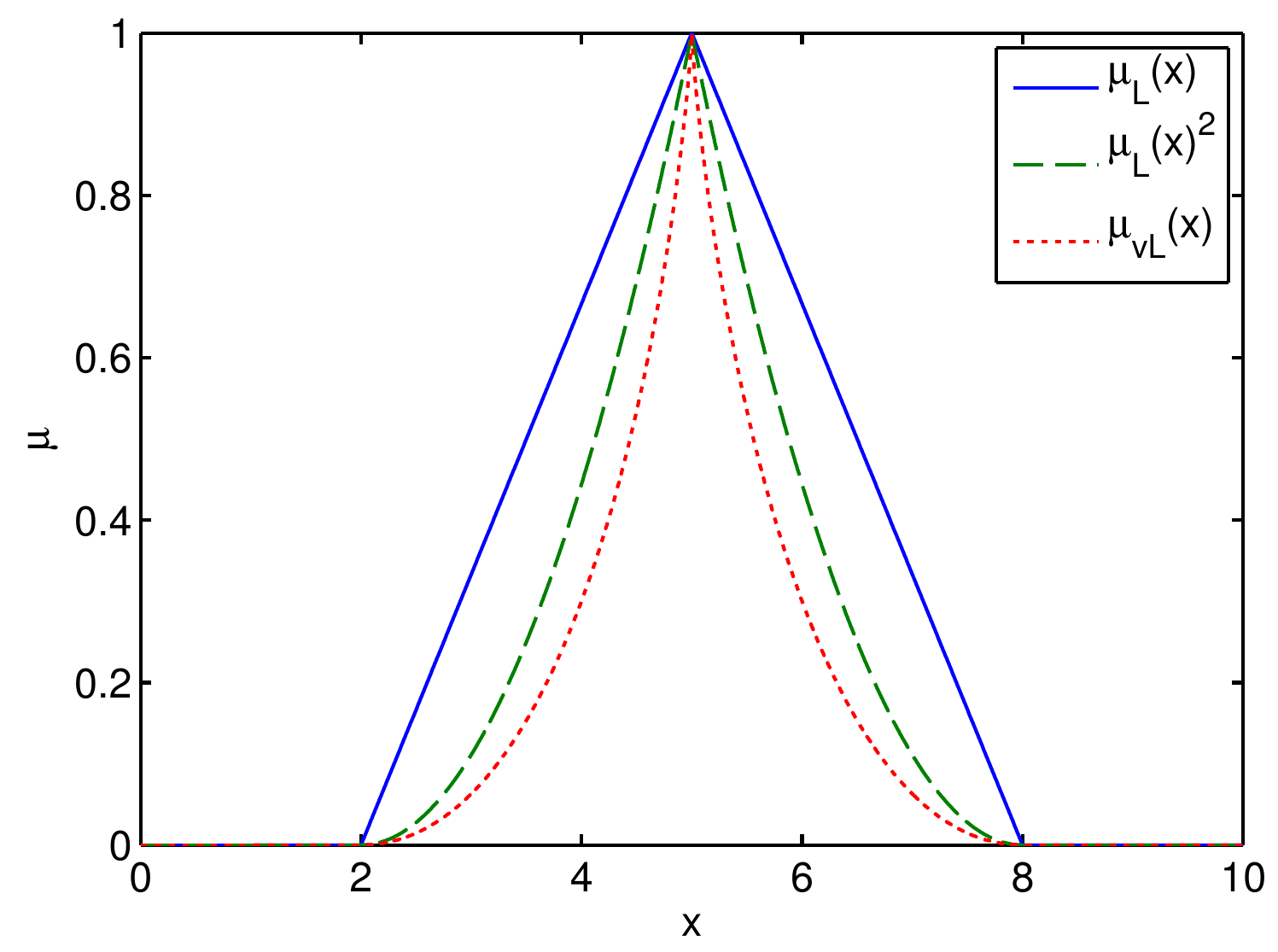}
\caption{Plots of $\mu_L(x)$, $\mu_L(x)^2$, and $\mu_{vL}(x) = \mu_L(x) + (1 - \mu_L(x))\ln(1- \mu_L(x))$. Values of $\mu_{vL}(x)$ decrease more quickly as we move from the prototype $P = 5$ than do equivalent values of $\mu_L(x)^2$.}
\label{fig:con_plot}
\end{figure}
\end{exa}

These hedges can also be applied across multiple dimensions, demonstrated in the example below.
\begin{exa}
Suppose we have two labels `tall' and `thin'. `Tall' has prototype $P_{\text{tall}} = 6.5$ft and `thin' has prototype $P_{\text{thin}} = 24$in. The appropriateness of each label is defined by:
\[ 
\mu_{\text{tall}}(x_1) = 
\begin{cases}
1 & if x_1 > 6.5\\
1 - (P_{\text{tall}} - x_1) & if x_1 \in [5.5, 6.5]\\
0 & \text{otherwise}
\end{cases}
\]

where the variable $x_1$ measures height

\[ 
\mu_{\text{thin}}(x_2) = 
\begin{cases}
1 & if x_2 < 24\\
1 - \frac{(x_2 - P_{\text{thin}})}{4} & if x_2 \in [24, 28]\\
0 & \text{otherwise}
\end{cases}
\]

where $x_2$ measures waist size.

Suppose further that being tall and being thin are independent of each other. The appropriateness of the label `tall and thin' could then be defined by:

\[
\mu_{\text{tall and thin}}(x_1, x_2) = \mu_{\text{tall}}(x_1)\mu_{\text{thin}}(x_2)
\]

If `tall' and `thin' are independent, we can treat their hedges separately, so the appropriateness of a label `very tall and quite thin' is:

\begin{align*}
\mu_{\text{very tall and quite thin}}(x_1, x_2) &= \mu_{\text{very tall}}(x_1)\mu_{\text{quite thin}}(x_2)\\
& = (\mu_{\text{tall}}(x_1)+ (1 - \mu_{\text{tall}}(x_1)\ln(1 - \mu_{\text{tall}}(x_1))) (\mu_{\text{thin}}(x_2) -  \mu_{\text{thin}}(x_2)\ln(\mu_{\text{thin}}(x_2)))
\end{align*}

This is illustrated in figure \ref{fig:vtqt}

\begin{figure}[htbp]
\centering
\includegraphics[scale = 0.5]{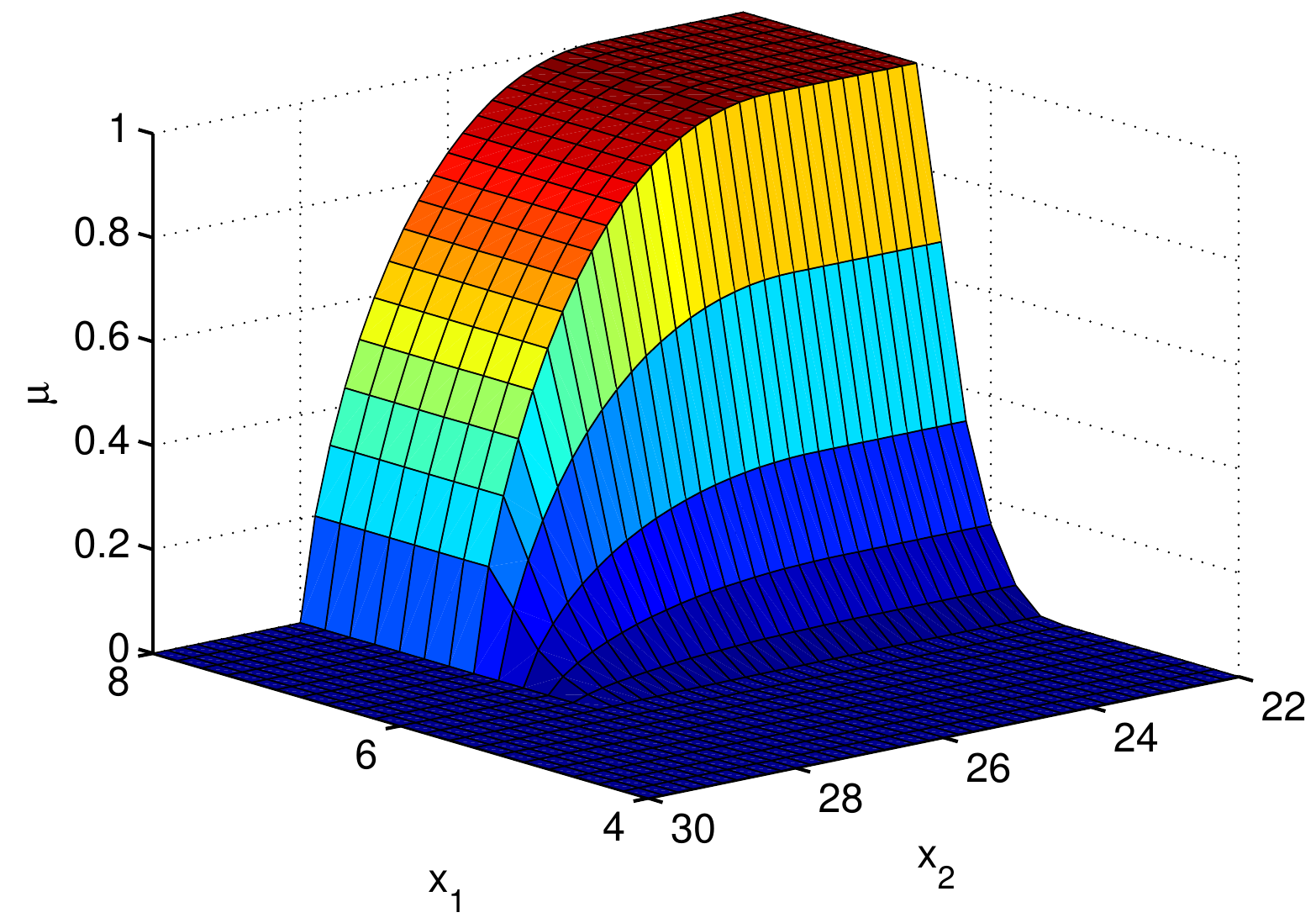}
\caption{Plot of $\mu_{\text{very tall and quite thin}}(x_1, x_2) = \mu_{\text{very tall}}(x_1)\mu_{\text{quite thin}}(x_2)$. Notice how appropriateness drops off quickly as the variable $x_1$, i.e. height, decreases.  In contrast, appropriateness decreases slowly as  we increase the variable $x_2$, or waist size}
\label{fig:vtqt}
\end{figure}
\end{exa}

\subsection{Hedges with differing prototypes}
\label{sub:diff}
As they stand, the hedges proposed leave the core and support of the fuzzy sets unchanged, which is often argued to be undesirable \cite{bosc, boscempty, novak, marin}. A slight modification yields models of hedges in which the core, or prototype, of the concept has been changed.

\begin{thm}[Dilation]
Suppose that $L_2 = $ quite $L_1$, as in theorem \ref{thm:dil}, but that $P_2 \neq P_1$. Then  $\mu_{L_2}(x) = \Delta_1(d(x,P_2)) - \Delta_1(d(x, P_2))\ln(\Delta_1(d(x, P_2)))$.
\end{thm}

\begin{proof}
 Substitute $\Delta_1(d(x, P_2))$ for $\mu_{L_1}(x)$ throughout proof of theorem \ref{thm:dil}
\end{proof}

\begin{thm}[Concentration]
Suppose that $L_2 = $ very $L_1$, as in theorem \ref{thm:con}, but that $P_2 \neq P_1$. Then  $\mu_{L_2}(x) = \Delta_1(d(x,P_2)) + (1 - \Delta_1(d(x, P_2)))\ln(1 - \Delta_1(d(x, P_2)))$.
\end{thm}

\begin{proof}
As above.
\end{proof}

\begin{cor}
If $\e_2 \geq \e_1$ and $P_2 \supseteq P_1$ then $\mu_{L_2}(x) \geq \mu_{L_1}(x) - \mu_{L_1}(x)\ln(\mu_{L_1}(x))$, and if $\e_2 \leq \e_1$ and $P_2 \subseteq P_1$, then $\mu_{L_2}(x) \leq \mu_{L_1}(x) + (1-\mu_{L_1}(x))\ln(1-\mu_{L_1}(x))$.
\end{cor}

\begin{proof}
 $\mu_{L_2}(x) = \Delta_1(d(x,P_2)) - \Delta_1(d(x, P_2))\ln(\Delta_1(d(x, P_2)))$, but since $P_2 \supseteq P_1$, $d(x, P_2) \leq d(x, P_1)$ $\forall x \in \Omega$, and so $\Delta_1(d(x, P_2)) \geq \Delta_1(d(x, P_1)) = \mu_{L_1}(x)$ $\forall x \in \Omega$. Hence, $\mu_{L_2}(x) \geq \mu_{L_1}(x) - \mu_{L_1}(x)\ln(\mu_{L_1}(x))$. A similar argument shows that $\mu_{L_2}(x) \leq \mu_{L_1}(x) + (1-\mu_{L_1}(x))\ln(1-\mu_{L_1}(x))$.
\end{proof}

\begin{exa}
\label{exa:exp_dil_plot}
Suppose our conceptual space $\Omega = \mathbb{R}$ with Euclidean distance and that a label $L$ has prototype $P = [4.5,5.5]$,  and threshold $\e \sim $ Uniform$[0,3]$. Then 
\[
\mu_L(x)  = 
\begin{cases}
1 & \text{ if } x \in [4.5,5.5]\\
1 - \frac{|x - 5|}{3} & \text{ if } x \in [1.5, 4.5] \text{ or } x \in [5.5, 8.5] \\
0 & otherwise
\end{cases}•
\]

We form the concept $qL$ by setting the prototype to be $P_q = [4,6]$, and $\e_q \geq \e$. The effect of applying our dilation hedge is illustrated in figure \ref{fig:exp_dil_plot}.

\begin{figure}[htbp]
\centering
\includegraphics[scale = 0.5]{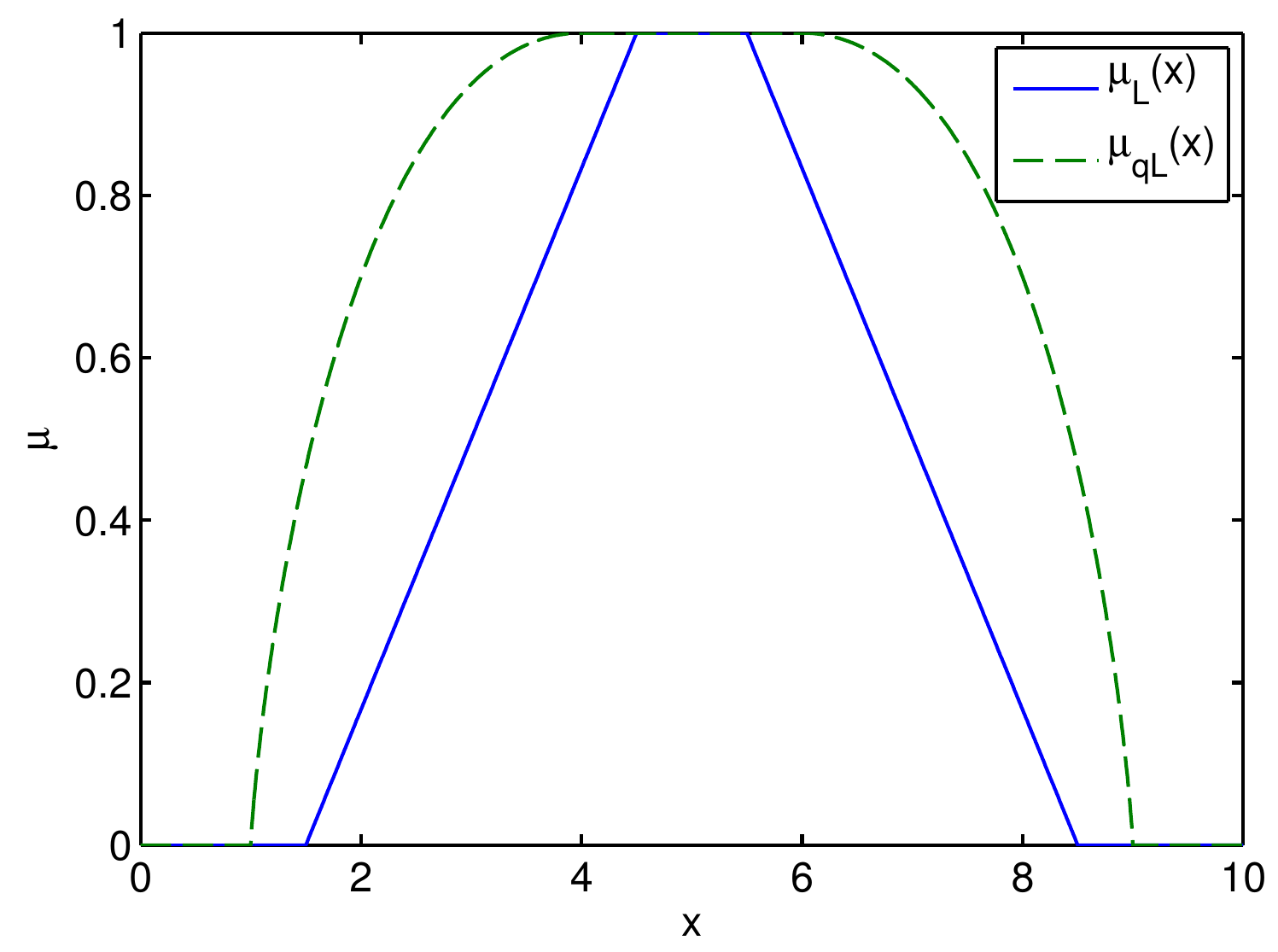}
\caption{Plots of $\mu_L(x)$ and $\mu_{qL}(x) = \Delta(d(x,P_q)) - \Delta(d(x, P_q))\ln(\Delta(d(x, P_q)))$. The prototype of $L$ has been expanded.}
\label{fig:exp_dil_plot}
\end{figure}

Conversely, suppose that a label $L$ has prototype $P = [4, 6]$ and threshold $\e \sim $ Uniform$[0,3]$. Then 
\[
\mu_L(x)  = 
\begin{cases}
1 & \text{ if } x \in [4,6]\\
1 - \frac{|x - 5|}{3} & \text{ if } x \in [1, 4] \text{ or } x \in [6, 9]\\
0 & otherwise
\end{cases}•
\]

We now form the concept $vL$ by contracting the prototype to $P_v = [4.5,5.5]$ and setting $\e_v \leq \e$. The effect of applying the contraction hedge is illustrated in figure \ref{fig:con_con_plot}.

\begin{figure}[htbp]
\centering
\includegraphics[scale = 0.5]{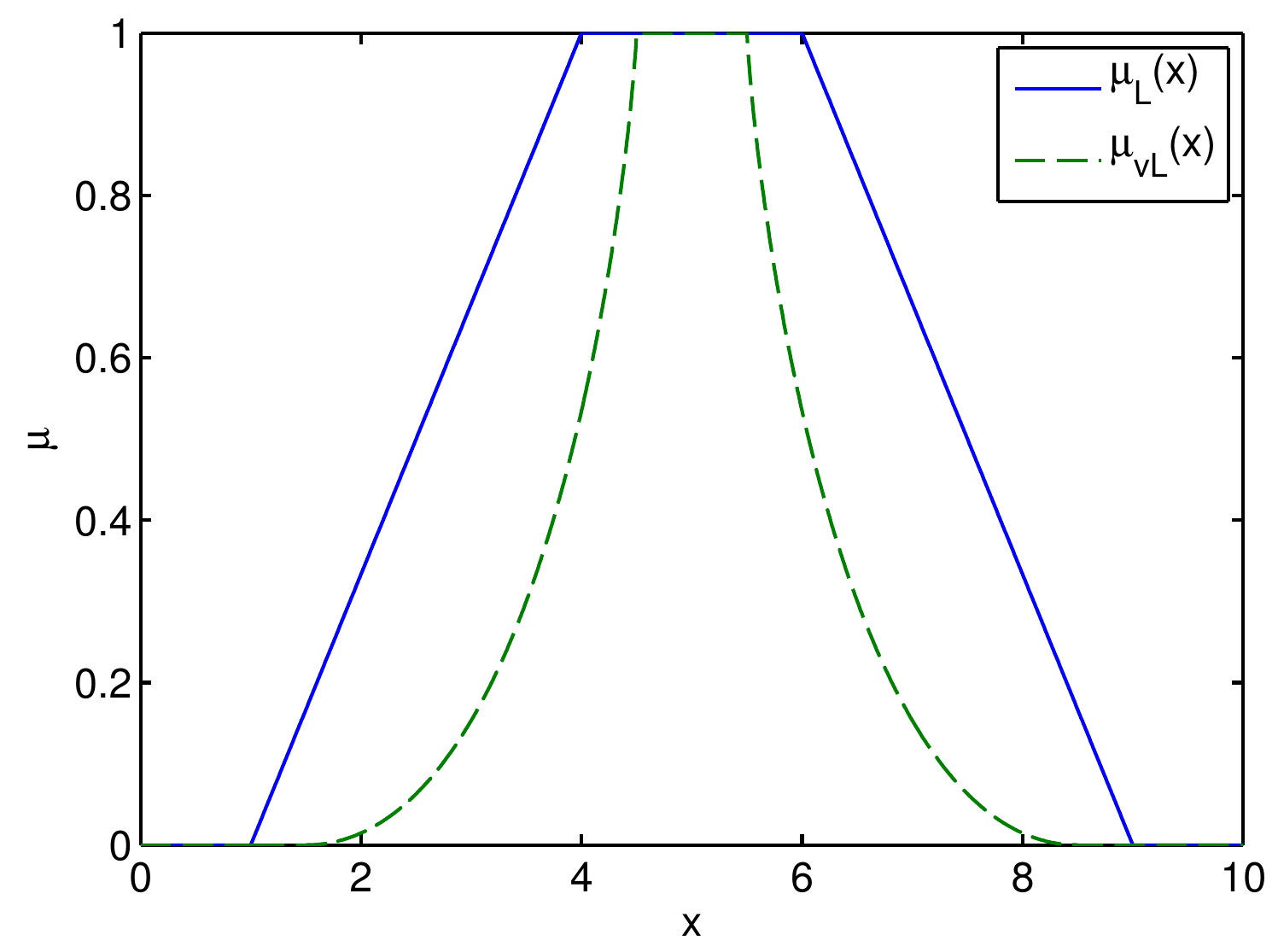}
\caption{Plots of $\mu_L(x)$ and $\mu_{vL}(x) = \Delta(d(x,P_v) + (1 - \Delta(d(x, P_v))\ln(1 - \Delta(d(x, P_v)))$. The prototype of $L$ has been reduced.}
\label{fig:con_con_plot}
\end{figure}

\end{exa}

\subsection{Functions of thresholds}
\label{sub:f}
It may be the case that the threshold of a given concept is greater than or less than a function of the original threshold. This could hold when a hedged concept is a very expanded or restricted version of the original concept, such as when the hedge `loosely' or `extremely' is used. Our formulae can also take account of this.

\begin{thm}
\label{thm:f}
Suppose $L_2 = <P_2, d, \delta_{\e_2}>$ is a dilation of $L_1 = <P_1, d, \delta_{\e_1}>$ with $P_2 \neq P_1$ and $\e_2 \geq f(\e_1)$,  where $f: \mathbb{R} \rightarrow \mathbb{R}$ is strictly increasing or decreasing. Then 
\[
\mu_{L_2}(x) = \Delta_1(f^{-1}(d(x,P_2))) - \Delta_1(f^{-1}(d(x,P_2)))\ln(\Delta_1(f^{-1}(d(x,P_2))))
\]
\end{thm}

\begin{proof}
Rewrite $\e_2 \geq f(\e_1)$ as $\e_2 \geq \e = f(\e_1)$, where $\e \sim \delta$ and is associated with a label $L$ with prototype $P$. Then:
\[
	\mu_{L_2}(x) = \Delta(d(x,P_2)) - \Delta(d(x, P_2))\ln(\Delta(d(x, P_2)))
\]
as above

 Since $f: \mathbb{R} \rightarrow \mathbb{R}$ is strictly monotone, $f^{-1}$ exists, and $\Delta(d(x, P)) = P(d(x,P) \leq \e) = P(f^{-1}(d(x,P)) \leq \e_1) = \Delta_1(f^{-1}(d(x,P)))$.

So
\[
\mu_{L_2}(x) = \Delta_1(f^{-1}(d(x,P_2))) - \Delta_1(f^{-1}(d(x,P_2)))\ln(\Delta_1(f^{-1}(d(x,P_2))))
\]

as required.
\end{proof}

\begin{thm}
Suppose $L_2 = <P_2, d, \delta_{\e_2}>$ is a concentration of $L_1 = <P_1, d, \delta_{\e_1}>$ with $P_2 \neq P_1$ and $\e_2 \leq f(\e_1)$,  where $f: \mathbb{R} \rightarrow \mathbb{R}$ is strictly increasing or decreasing. Then 
\[
\mu_{L_2}(x) = \Delta_1(f^{-1}(d(x,P_2))) +(1- \Delta_1(f^{-1}(d(x,P_2))))\ln(1 -\Delta_1(f^{-1}(d(x,P_2))))
\]
\end{thm}

\begin{proof}
The proof is entirely similar to that of theorem \ref{thm:f}
\end{proof}

\subsection{Links to other models of hedges}
\label{sub:det}
It is possible to specify the dependence of the threshold of the hedged concept on the threshold of the unhedged concept purely deterministically, i.e. by $\e_2 = f(\e_1)$, rather than  $\e_2 \leq f(\e_1)$. In this case, we can show links to other models of hedges from the literature.

A simple example of a deterministic dependency is given below.

\begin{exa}
 Suppose $\Omega = \mathbb{R}$, $d$ is Euclidean distance and that $L_1$ has prototype $P = 5$ and $\e \sim $ Uniform$[0,3]$. Then as before, 

\[
\mu_L(x)  = 
\begin{cases}
1 - \frac{|x - 5|}{3} & \text{ if } x \in [2, 8]\\
0 & otherwise
\end{cases}•
\]

To implement a dilation hedge, we would form a new label $qL$ with $P_q = P = 5$ and $\e_q = k_q\e$ with $k_q > 1$. For a contraction hedge, we would form the label $vL$ by setting $P_v = P = 5$ and $\e_v = k_v\e$ with $k_v <1$. Then, 

\[
\mu_{hL}(x)  = 
\begin{cases}
1 - \frac{|x - 5|}{3k} & \text{ if } x \in [5 - 3k, 5+ 3k]\\
0 & otherwise
\end{cases}•
\]
where $h = q$ or $v$, $k = k_q$ or $k_v$ respectively.

The effect of implementing these hedges is illustrated in figure \ref{fig:f_plot}.

\begin{figure}[htbp]
\centering
\includegraphics[scale = 0.5]{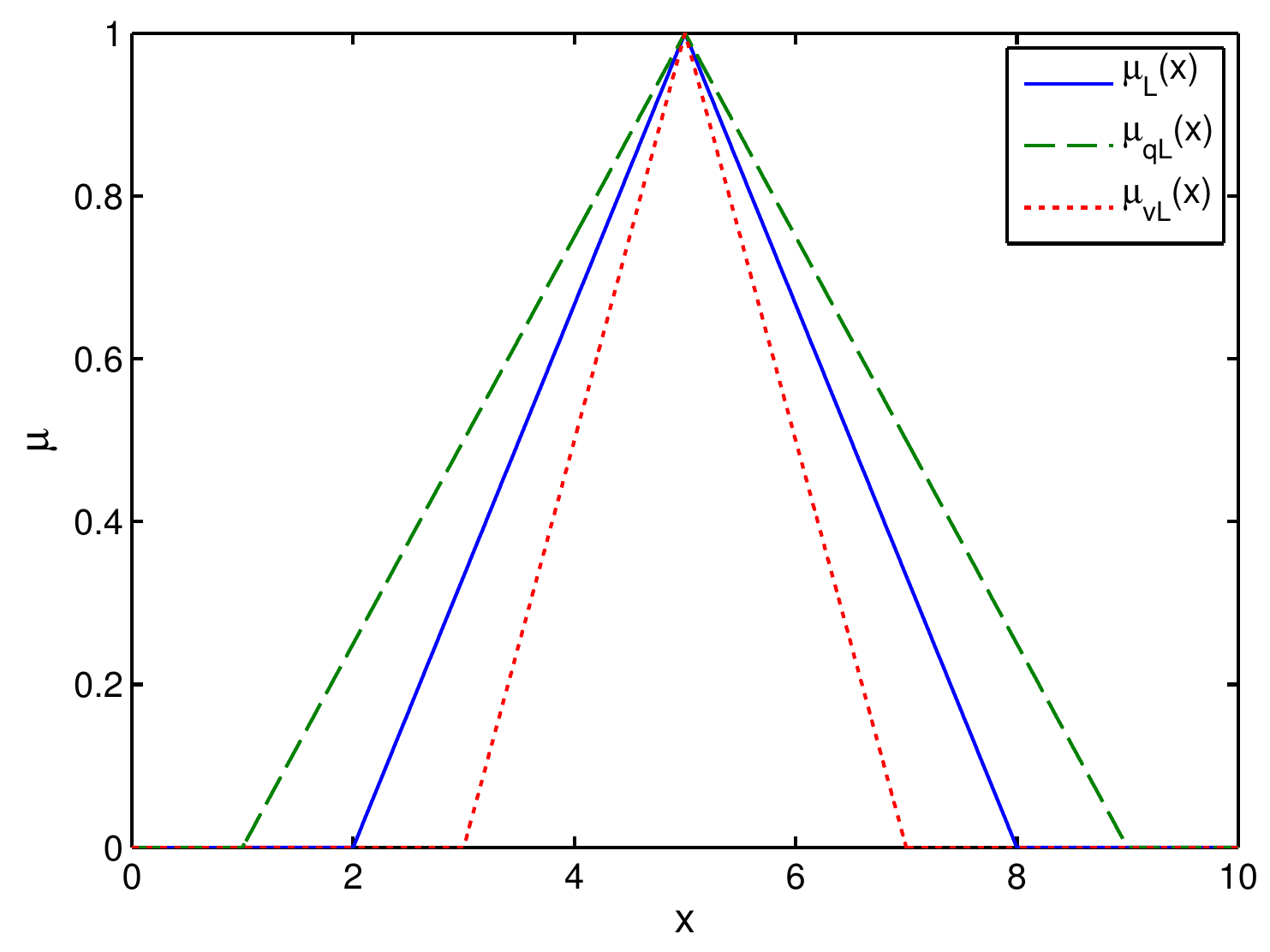}
\caption{Plots of $\mu_L(x)$, $\mu_{qL}(x)$ and $\mu_{vL}(x)$. The prototype, a single point, remains constant on application of the hedges.}
\label{fig:f_plot}
\end{figure}

\end{exa}

Using this approach, we can also create an effect similar to that of changing the prototype.  Suppose that a label $L$ in a conceptual space $\Omega$ has a single point $P$ as a prototype, but that the minimum value of the threshold $\e$ is greater than 0, for example, $\e \sim$ Uniform$[c,a]$. Then 

\[
\mu_{L}(x)  = 
\begin{cases}
1 & \text{ if } d(x,P) < c\\
\frac{a}{a - c} - \frac{|x - P|}{a - c} & \text{ if } d(x, P) \in [a,c]\\
0 & otherwise
\end{cases}•
\]

Suppose that a hedged concept $hL$ is formed from $L$ by the dependency $\e_h = k\e$ where $k$ is a constant. Then 

\[
\mu_{L}(x)  = \Delta(\frac{|x-P|}{k}) = 
\begin{cases}
1 & \text{ if } \frac{|x-P|}{k} < c\\
\frac{a}{(a - c)} - \frac{|x - P|}{k(a - c)} & \text{ if } \frac{|x-P|}{k} \in [a,c]\\
0 & otherwise
\end{cases}• = 
\begin{cases}
1 & \text{ if } |x - P| < kc\\
\frac{ka - |x - P|}{k(a - c)} & \text{ if } |x-P| \in [ka,kc]\\
0 & otherwise
\end{cases}•
\]

This effect is illustrated in the example below.

\begin{exa}
Suppose that the conceptual space $\Omega = \mathbb{R}$ and that a label $L$ has prototype $P = 5$ and threshold $\e \sim $ Uniform$[1, 2]$. Then 

\[
\mu_{L}(x)  = 
\begin{cases}
1 & \text{ if } |x - 5| < 1\\
2 - |x - 5| & \text{ if } |x - 5| \in [1,2]\\
0 & otherwise
\end{cases}•
\]

Forming a new label $qL$ by applying the hedge $\e_q = 2\e$ gives appropriateness measure

\[
\mu_{qL}(x)  = 
\begin{cases}
1 & \text{ if } |x - 5| < 2\\
\frac{4 - |x - 5|}{2} & \text{ if } |x - 5| \in [2,4]\\
0 & otherwise
\end{cases}•
\]

Forming a new label $vL$ by applying the hedge $\e_v = 0.5\e$ gives appropriateness measure 

\[
\mu_{vL}(x)  = 
\begin{cases}
1 & \text{ if } |x - 5| < 0.5\\
\frac{1 - |x - 5|}{0.5} & \text{ if } |x - 5| \in [0.5,1]\\
0 & otherwise
\end{cases}•
\]

These are illustrated in figure \ref{fig:fprot_plot}.

\begin{figure}[htbp]
\centering
\includegraphics[scale = 0.5]{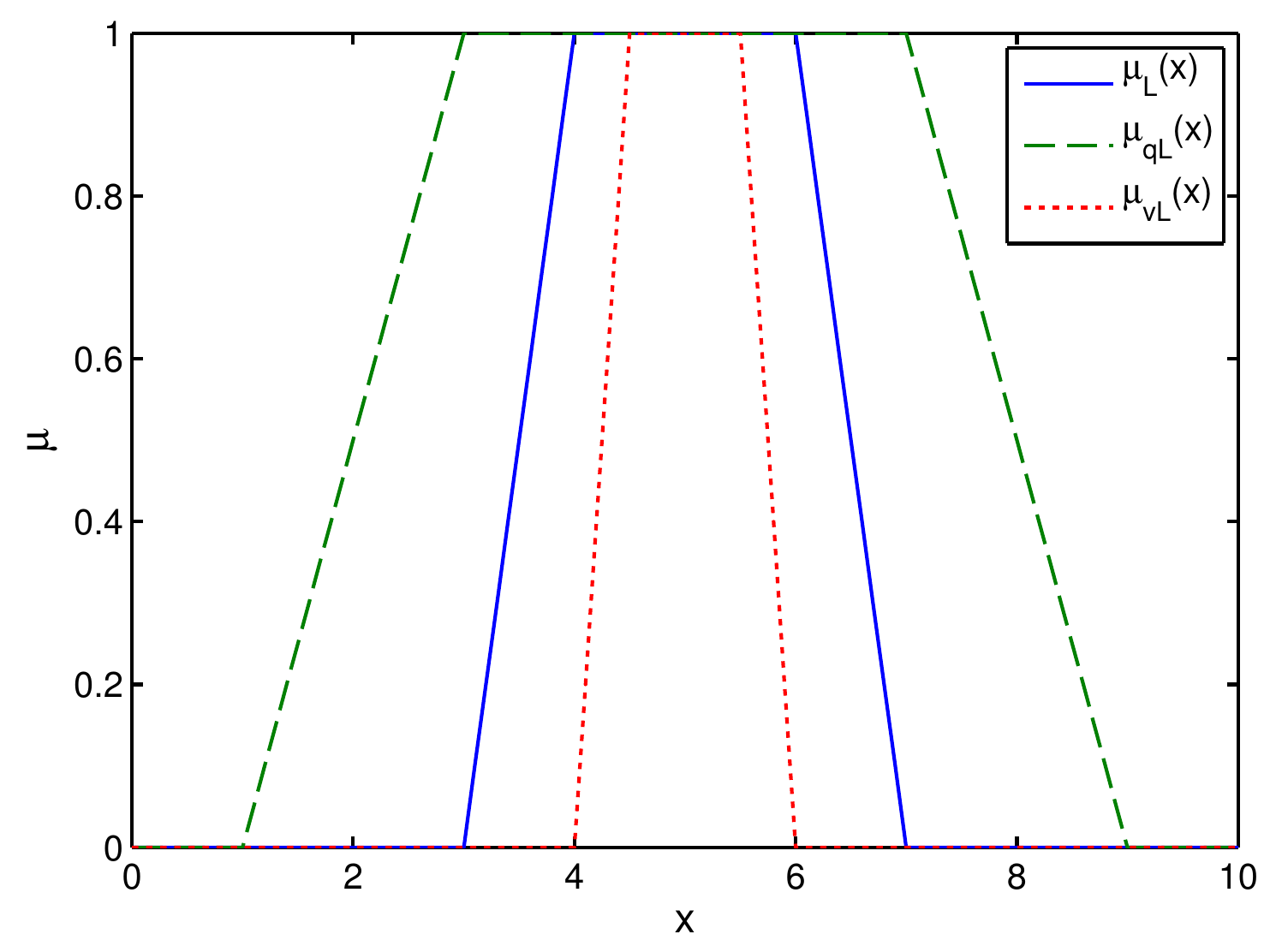}
\caption{Plots of $\mu_L(x)$, $\mu_{qL}(x)$ and $\mu_{vL}(x)$. In this case, the set $\{x \in \Omega: \mu_{hL}(x) = 1\}$, where $h = v, q\text{ or nothing}$, expands or contracts on application of the hedges.}
\label{fig:fprot_plot}
\end{figure}
\end{exa}

Notice that if we set $\Omega = \mathbb{R}^+$ and label $L$ specified by $P = 0$, $\e \sim $ Uniform$[c, a]$, this is identical to the linear membership model given in \cite{novak}. Specifically, we have 

\[
\mu_{L}(x)  = 
\begin{cases}
1 & \text{ if } x < c\\
\frac{a - x}{a-c} & \text{ if } x \in [c,a]\\
0 & \text{otherwise}
\end{cases}
\]

Forming a hedged concept $hL$ by setting $P_{hL}=P=0$ and $\e_{hL} = k\e$ gives

\[
\mu_{hL}(x)  = 
\begin{cases}
1 & \text{ if } x < kc\\
\frac{ka - x}{ka-kc} & \text{ if } x \in [kc,ka]\\
0 & \text{otherwise}
\end{cases}
\]

Comparing this with the model given in section \ref{sec:lh}, we see that this is precisely the model proposed by \cite{novak} in the linear case.

Similarity between the hedging effects illustrated in figure \ref{fig:f_plot} and the effects implemented in the model proposed in \cite{bosc}, illustrated in figure \ref{fig:bosc_plot}, can clearly be seen. To derive the model given in \cite{bosc}, we describe the fuzzy sets associated with labels $L$ and $hL$ in trapezoidal notation. Notice that $L= (P-a, P - c, P+c, P+a)$ and $hL = (P-ka, P- kc, P+kc, P+ka)$. We can render this transformation in the terms employed by \cite{bosc}. Consider labels $L$ and $hL$ as fuzzy sets characterised by the appropriateness measure $\mu_{L}(x)$ and $\mu_{hL}(x)$. Then, in the case of dilation, we have:

\[
qL = E^Z(L)(s) = \text{sup}_{r \in \Omega}T(\mu_{L}(s), E^Z(s,r))
\]

and for contraction,

\[
vL = E_Z(L)(s) = \text{inf}_{r \in \Omega}I(\mu_{L}(s), E^Z(s,r))
\]

When $T$ is the $T$-norm $min$, I is the G\"{o}del implication and  $Z = (-z-\alpha, -z, z, z+\alpha)$ (with the restriction that $c < z$ to ensure a well-defined set), the approach in \cite{bosc} gives $qL = (P-a - z - \alpha, P - c - z, P+c + z, P+a + z + \alpha)$. If we set $z = (k - 1)c$ and $\alpha = (k - 1)(a - c)$, this is equal to  $qL = (P-ka, P- kc, P+kc, P+ka)$.

However, we also require that $vL = (P-ka, P- kc, P+kc, P+ka)$.  The approach in \cite{bosc} gives $vL = (P-a + z + \alpha, P - c + z, P+c - z, P+a - z - \alpha)$, and we therefore need to set  $z = (1-k)c$ and $\alpha = (1-k)(a-c)$

This formulation is not as general as given in \cite{bosc}, however, note that it only uses one additional parameter and no additional operators, rather than the two parameters and either a $T$-norm or implication used by \cite{bosc}.

Two more key models from the literature are the powering and shifting modifiers proposed in \cite{zadehhedges}. Recall that powering modifiers are of the form $\mu_{hL}(x) = (\mu_L(x))^k$ and shifting modifiers are of the form $\mu_{hL}(x) = (\mu_L(x-a))$. Shifting modifiers are easy to implement within our model, simply by shifting the prototype by the quantity $a$.

Powering modifiers can be expressed as a function of the threshold $\e$ given a particular distribution of the threshold $\delta$. Suppose $\Omega = \mathbb{R}$, $\e \sim U[0,c]$, giving

\[
\mu_L(x)  = 
\begin{cases}
1 - \frac{d(x, P)}{b} & \text{ if } x \in [P-b, P+b]\\
0 & otherwise
\end{cases}•
\]

and suppose a new label $hL$ is formed with prototype $P$ and threshold $\e_h = f(\e)$ such that $\mu_{hL}(x) = \mu_L(x)^k$. Then $\mu_{hL}(x) = \Delta(f^{-1}(d(x,P))) = (\Delta(d(x,P)))^k$, so

\[
f^{-1}(d(x,P)) = \Delta^{-1}((\Delta(d(x,P)))^k) = b - b(\frac{(b - d(x, P))^k}{b^k}) = b - \frac{(b - d(x,P))^k}{b^{k-1}}
\]

and hence 
\[
\e_{hL} = f(\e) = b - (b^{k-1}(b - \e))^{1/k}
\]

This expression seems surprisingly complicated, and there may be better ways of deriving the powering hedges that are not as a function of the threshold $\e$.

In this section we have shown that our general model can capture some of the many approaches found in the literature as special cases.  We now go on to look at the property of compositionality that is exhibited by a number of models.

\subsection{Compositionality}
\label{sub:comp}
One of the features of hedges seen in \cite{bosc, decock2002, novak, zadehhedges} is that they can be applied multiple times. Within the label semantics framework, this consists in expanding or reducing the threshold of a concept a number of times. The directed acyclic graph corresponding to this is shown  in figure \ref{fig:compn}. We show below that expressions for `very' and `quite' as given in theorems \ref{thm:dil} and \ref{thm:con} are compositional, and that the appropriateness of a concept after $n$ applications of a hedge can be expressed purely in terms of the appropriateness after $n -1$ applications. We also derive expressions for the composition of deterministic hedges as described in section \ref{sub:det}.

\begin{figure}
  \centering
  \includegraphics[scale = 1]{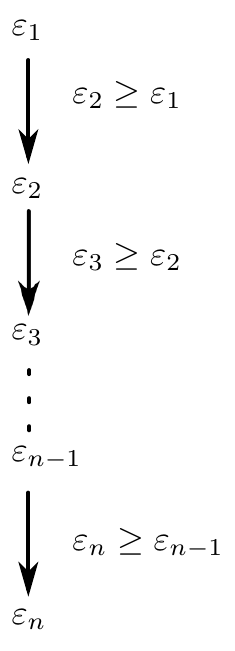}
	  \caption{Threshold dependencies on multiple applications of the hedge `quite'. Each threshold $\e_i$ is directly dependent on the one preceding it, $\e_{i-1}$.}
	\label{fig:compn}
\end{figure}

\begin{thm}
\label{thm:comp}
Suppose that labels $L_1, L_2, ... , L_n$ are defined by prototypes $P_1 = P_2 = ... = P_n = P$, thresholds $\e_1 \geq \e_2 \geq ... \geq \e_n$ and with a distance metric $d$ common to all labels. Then $\mu_{L_n}(x) = \mu_{L_{n-1}}(x)  +  (1-\mu_{L_{n-1}}(x))\ln(1 - \mu_{L_{n-1}}(x))$
\end{thm}

\begin{proof}

We proceed by induction on $n$. Theorem \ref{thm:dil} proves this for $n=2$. Assuming true for $n=k$, we have

\begin{align*}
	\mu_{L_{k+1}}(x) &= \int_0^\infty \int_0^\infty ... \int_{max(d(x,P), \e_k)}^\infty \delta(\e_1, \e_2,...,\e_{k+1}) \mathrm{d}\e_{k+1}...\mathrm{d}\e_1 \\
	&= \int_0^\infty \frac{\delta_{\e_1}(\e_1)}{\Delta_1(\e_1)} \int_0^\infty \frac{\delta_{\e_1}(\e_2)}{\Delta_2(\e_2)}  ... \int_0^\infty \frac{\delta_{\e_{k-1}}(\e_k)}{\Delta_k(\e_k)} \int_{max(d(x,P), \e_{k})}^\infty \delta_{\e_k}(\e_{k+1}) \mathrm{d}\e_{k+1}...\mathrm{d}\e_1 \\
	&=\int_0^\infty \frac{\delta_{\e_1}(\e_1)}{\Delta_1(\e_1)} \int_0^\infty \frac{\delta_{\e_1}(\e_2)}{\Delta_2(\e_2)}  ... \int_{max(d(x, P), \e_{k-1})}^\infty \frac{\delta_{\e_{k-1}}(\e_k)}{\Delta_k(\e_k)} \overbrace{\int_{\e_k}^\infty \delta_{\e_k}(\e_{k+1}) \mathrm{d}\e_{k+1}}^{= \Delta_k(\e_k)} \mathrm{d}\e_k...\mathrm{d}\e_1 \\ 
	& \quad +\int_0^{d(x,P)} \frac{\delta_{\e_1}(\e_1)}{\Delta_1(\e_1)} \int_{\e_1}^{d(x,P)} \frac{\delta_{\e_1}(\e_2)}{\Delta_2(\e_2)}  ... \int_{\e_{k-1}}^{d(x,P)} \frac{\delta_{\e_{k-1}}(\e_k)}{\Delta_k(\e_k)} \overbrace{\int_{d(x,P)}^\infty \delta_{\e_k}(\e_{k+1}) \mathrm{d}\e_{k+1}}^{= \mu_{L_k}(x)} \mathrm{d}\e_k...\mathrm{d}\e_1 \\
	&= \overbrace{\int_0^\infty \frac{\delta_{\e_1}(\e_1)}{\Delta_1(\e_1)} \int_0^\infty \frac{\delta_{\e_1}(\e_2)}{\Delta_2(\e_2)}  ... \int_{max(d(x, P), \e_k)}^\infty \delta_{\e_{k-1}}(\e_k)\mathrm{d}\e_k...\mathrm{d}\e_1 }^{= \mu_{L_k}(x) \text{ by ind. hyp.}}\\
	& \quad + \mu_{L_k}(x) \int_0^{d(x,P)} \frac{\delta_{\e_1}(\e_1)}{\Delta_1(\e_1)} \int_{\e_1}^{d(x,P)} \frac{\delta_{\e_1}(\e_2)}{\Delta_2(\e_2)}  ... \int_{\e_{k-1}}^{d(x,P)} \frac{\delta_{\e_{k-1}}(\e_k)}{\Delta_k(\e_k)} \mathrm{d}\e_k...\mathrm{d}\e_1\\
	&= \mu_{L_k}(x) + \mu_{L_k}(x) \int_0^{d(x,P)} \frac{\delta_{\e_{k-1}}(\e_k)}{\Delta_k(\e_k)} ... \int_0^{\e_3} \frac{\delta_{\e_1}(\e_2)}{\Delta_2(\e_2)} \int_0^{\e_2} \frac{\delta_{\e_1}(\e_1)}{\Delta_1(\e_1)} \mathrm{d}\e_1...\mathrm{d}\e_k\\
	&= \mu_{L_k}(x) + \mu_{L_k}(x) \overbrace{\int_0^{d(x,P)} \frac{\delta_{\e_{k-1}}(\e_k)}{\Delta_k(\e_k)} ... \int_0^{\e_3} \frac{-\delta_{\e_1}(\e_2)\ln(\Delta_1(\e_2))}{\Delta_2(\e_2)} \mathrm{d}\e_2...\mathrm{d}\e_k}^{=A}
\end{align*}

By the inductive hypothesis, $\forall i = 0...k$ 
\begin{align*}
	\delta_{\e_i}(\e_i) &= -\frac{\mathrm{d}}{\mathrm{d\e_i}}\Delta_i(\e_i)\\ \nonumber
	&= -\frac{\mathrm{d}}{\mathrm{d\e_i}}(\Delta_{i-1}(\e_i) - \Delta_{i-1}(\e_i)\ln(\Delta_{i-1}(\e_i)))\\
	&= -\delta_{\e_{i-1}}(\e_i)\ln(\Delta_{i-1}(\e_i))
\end{align*}

Recursively substituting in $A$, we obtain

\begin{align*}
	\mu_{L_{k+1}}(x) &= \mu_{L_k}(x) + \mu_{L_k}(x) \int_0^{d(x,P)} \frac{\delta_{\e_{k-1}}(\e_k)}{\Delta_k(\e_k)} ... \int_0^{\e_3} \frac{\delta_{\e_2}(\e_2)}{\Delta_2(\e_2)} \mathrm{d}\e_2...\mathrm{d}\e_k\\
	&= \mu_{L_k}(x) +\mu_{L_k}(x) \int_0^{d(x, P)}\frac{\delta_{\e_k}(\e_k)}{\Delta_k(\e_k)} \mathrm{d}\e_k\\
 	&= \mu_{L_k}(x) -\mu_{L_k}(x)\ln(\mu_{L_k}(x))
\end{align*}
\end{proof}

\begin{thm}
Suppose labels $L_1, L_2, ..., L_n$ are defined by prototypes $P_1 = P_2 = ... = P_n = P$, thresholds $\e_1 \leq \e_2 \leq ... \leq \e_n$, and that distance metric $d$ is common to all. Then $\mu_{L_n}(x) = \mu_{L_{n-1}}(x)  -  \mu_{L_{n-1}}(x)\ln(\mu_{L_{n-1}}(x))$
\end{thm}

\begin{proof}
Similar to proof of theorem \ref{thm:comp}.
\end{proof}

We can also derive expressions for the composition of deterministic hedges. 

\begin{thm}
 Suppose labels $L_1, L_2, ..., L_n$ are defined by prototypes $P_1 = P_2 = ... = P_n = P$, thresholds $\e_n = f(\e_{n-1}), \e_{n-1} = f(\e_{n-2}), ..., \e_2 = f(\e_1)$, where $f$ is monotone increasing or decreasing, and that distance metric $d$ is common to all. Then $\mu_{L_n}(x) = \Delta_1(f^{-(n-1)}(d(x,P))$, where $f^{-k}$ signifies $f^{-1}$ composed $k$ times.
\end{thm}

\begin{proof}
 $\mu_{L_2}(x) = \Delta_1(f^{-1}(d(x,P))$. Suppose that $\mu_{L_k}(x) = \Delta_k(d(x,P)) = \Delta_1(f^{-(k-1)}(d(x,P))$. Since $\e_{k+1} = f(\e_k)$, we have $\mu_{L_{k+1}}(x) = \Delta_k(f^{-1}(d(x,P)) = \Delta_1(f^{-k}(d(x,P)))$. 

Therefore $\mu_{L_n}(x) = \Delta_1(f^{-(n-1)}(d(x,P))$ by induction.
\end{proof}

Since labels can be composed in this way, we can model different degrees of emphasis corresponding to the composition of multiple hedges. So, for example, we could model `extremely L' as `very, very L'. This is illustrated in example \ref{exa:extremely}.

\begin{exa}
\label{exa:extremely}
Suppose the label $L$ is as described in example \ref{exa:dil_plot}, i.e.  $L$ has prototype $P = 5$,  and threshold $\e \sim $ Uniform$[0,3]$. We can form a new label $vL$ with prototype $P_v = P = 5$ and threshold $\e_v \leq \e$, which has appropriateness $\mu_{vL}(x) = \mu_L(x) + (1 - \mu_L(x))\ln(1- \mu_L(x))$ as shown in theorem \ref{thm:con}. We may then form another new label $vvL$ with prototype $P_{vv} = P_v = 5$ and threshold $\e_{vv} \leq \e_v$ with appropriateness $\mu_{vvL}(x) = \mu_{vL}(x) + (1 - \mu_{vL}(x))\ln(1- \mu_{vL}(x))$ as described in theorem \ref{thm:comp}. The effect of applying this contraction hedge is seen in figure \ref{fig:concon_plot}. We have contrasted the effect of the composed hedges with $(\mu_L(x))^4$.

\begin{figure}[htbp]
\centering
\includegraphics[scale = 0.5]{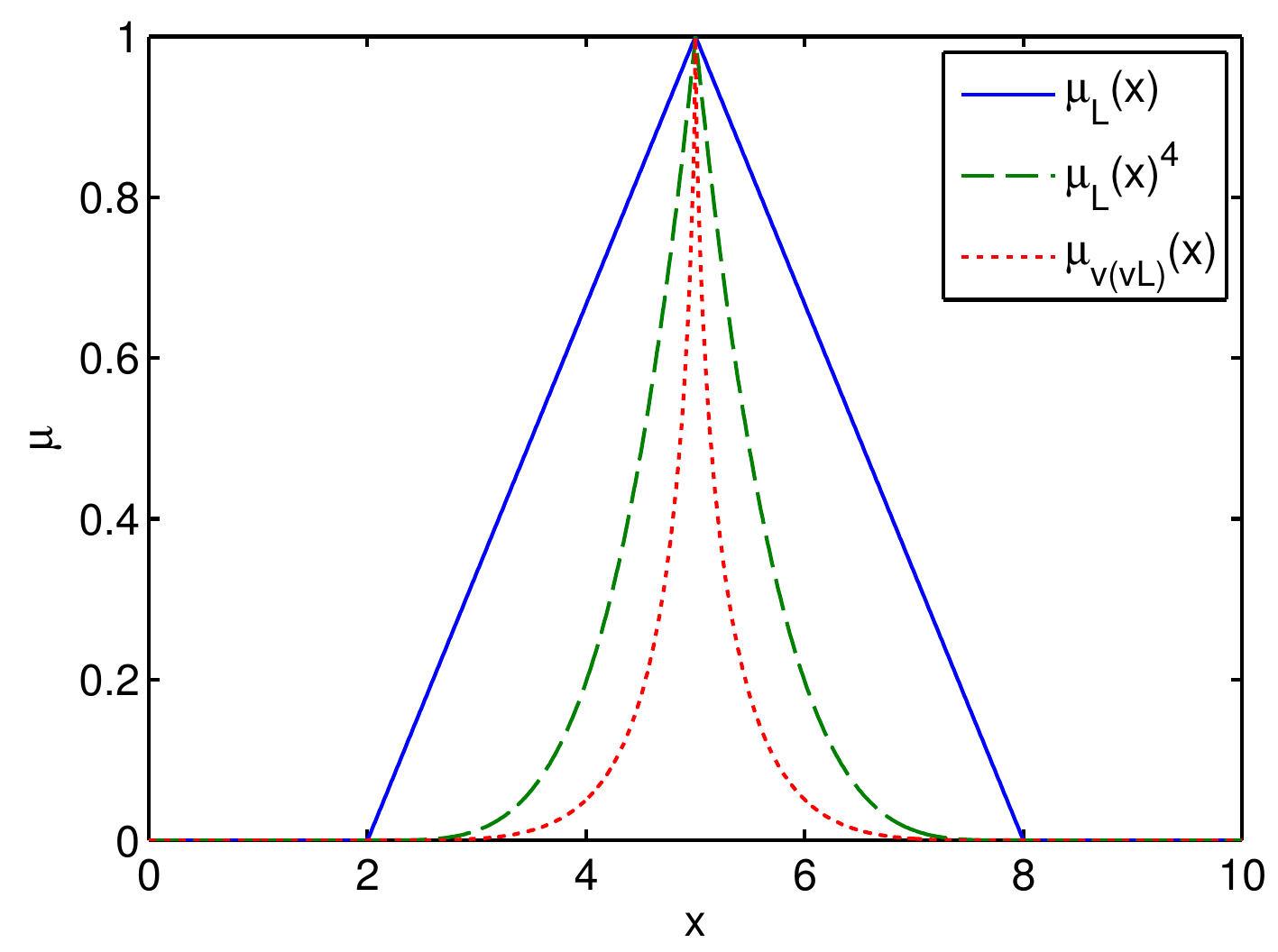}
\caption{Plots of $\mu_L(x)$, $\mu_L(x)^4$, and $\mu_{vvL}(x)$. Values of $\mu_{vvL}(x)$ decrease more quickly as we move from the prototype $P = 5$ than do equivalent values of $\mu_L(x)^4$.}
\label{fig:concon_plot}
\end{figure}
\end{exa}

Since these hedges can be composed only an integral number of times, we cannot obtain the differences in grade that could be achieved with using various powers in a powering modifier, e.g. $\mu_L(x)^{1.73}$. However, in section \ref{sub:f} we discuss how to tune the intensity of hedges by using dependencies on functions of thresholds.  We have further shown in section \ref{sub:det} how to derive powering and shifting modifiers within our framework. It would be interesting to explore how other examples of hedges can be expressed in this framework.

We have shown that when multiple hedges of the forms seen in theorems \ref{thm:dil} and \ref{thm:con} are used, $\mu_{L_n}(x)$ can be expressed purely in terms of the appropriateness of the label directly preceding it. We have not been able to find a closed form solution for this recurrence, however, we can  investigate the fixed points of the recurrence and examine what happens to the values of $\mu_{L_n}(x)$ as $n \rightarrow \infty$. We have also shown that deterministic hedges can be composed, and we go on to look at their behaviour in the limit of composition.

\subsection{Limits of Compositions.}
\label{sub:limits}
The following results examine the behaviour of  $\mu_{L_n}(x)$ as $n \rightarrow \infty$ 
\begin{thm}
\label{thm:dilcom}
Suppose $L_1, ..., L_n$ are labels obtained by repeated application of the dilation operator. Then  $\mu_{L_n}$  has a limit $M^+$ and $M^+ = 1$ $\forall x \in \Omega$ such that $\mu_{L_1}(x) \neq 0$, and $M^+ = 0$ otherwise.
\end{thm}

\begin{proof}

 $\mu_{L_{i+1}}(x) = \mu_{L_i}(x) - \mu_{L_i}(x)\ln(\mu_{L_i}(x))$, $i = 1,.., n-1$. If $\mu_{L_1}(x) = 1$ then $\mu_{L_i}(x) = 1$ $\forall i = 1,..., n$. Also, if $\mu_{L_1}(x) = 0$ then $\mu_{L_i}(x) = 0$ $\forall i = 1,..., n$. Suppose $\mu_{L_i}(x) \in (0,1)$. Then $\mu_{L_{i+1}}(x) > \mu_{L_i}(x)$, and so for $\mu_{L_1}(x) \in (0,1)$, $\mu_{L_1}(x) <...< \mu_{L_n}(x)$ is a strictly increasing sequence.

If a limit $M^+$ exists, then  we will have $M^+ = M^+ - M^+\ln(M^+)$, so either $M^+ = 0$  or $\ln(M^+) = 0$. We can't have $M^+ = 0$, since we assume that $\mu_{L_1}(x) \in (0,1)$ and the sequence is strictly increasing. Therefore, we must have $\ln(M^+) = 0$ and therefore $M^+ = 1$. 
So 

\[
	\mu_{L_\infty}(x) = \left\{
	\begin{array}{l l}
		1 & \quad \mu_{L_1}(x) \in (0,1]\\
		0 & \quad \mu_{L_1}(x) = 0\\
	\end{array} \right.
\]
\end{proof}

\begin{thm}
\label{thm:concom}
Suppose $L_1,..., L_n$ are labels obtained by repeated application of the contraction operator. Then $\mu_{L_n}$  has a limit $M^-$ and $M^- = 0$ $\forall x \in \Omega$ such that $\mu_{L_1}(x) \neq 1$, and $M^- = 1$ otherwise.
\end{thm}

\begin{proof}
 $\mu_{L_{i+1}}(x) = \mu_{L_i}(x) + (1- \mu_{L_i}(x))\ln(1 - \mu_{L_i}(x))$, $i = 1,..., n-1$. Again, if $\mu_{L_1}(x) = 1$ then $\mu_{L_i}(x) = 1$ $\forall i = 1,..., n$. Also, if $\mu_{L_1}(x) = 0$ then $\mu_{L_i}(x) = 0$ $\forall i = 1,..., n$, and for $\mu_{L_1}(x) \in (0,1)$, $\mu_{L_1}(x) >...> \mu_{L_n}(x)$ is a strictly decreasing sequence.

If a limit $M^-$ exists, then
\begin{align*}
M^- &= M^- + (1 - M^-)\ln(1- M^-)\\ 
\ln(1 - M^-) &= M^-\ln(1- M^-)
\end{align*}

So either $M^- = 0$ or $\ln(1- M^-) = 0$. If $\ln(1- M^-) = 0$ then $M^- = 1$, which is impossible since $\mu_{L_1}(x) \in (0,1)$ and the sequence of $\mu_{L_i}(x)$ is strictly decreasing. Therefore

\[
	\mu_{L_\infty}(x) = \left\{
	\begin{array}{l l}
		0 & \quad \mu_{L_1}(x) \in [0,1)\\
		1 & \quad \mu_{L_1}(x) = 1\\
	\end{array} \right.
\]
\end{proof}

We have shown here that in the limit, the result of applying dilation or contraction modifiers multiple times is to create a crisp set. In the case of dilation, the crisp set includes the whole support of the fuzzy set associated with the original label, whereas in the case of contraction, the concept reduces to include only its prototype.

When deterministic hedges are used, i.e. $\e_2 = f(\e_1)$, the behaviour of the limit depends on the behaviour of the function $f$ and its properties in the limit as $n \rightarrow \infty$ of $f^{-n}$.

\begin{exa}
\label{exa:detcom}
 Suppose $f(\e) = 0.5\e$. Applying this hedge multiple times will result in $\mu_{L_n}(x) = \Delta_1(2^n d(x,P))$. As $n\rightarrow \infty$, $2^n d(x,P) \rightarrow \infty$, except where $d(x,P)=0$. Therefore, 
\[
	\mu_{L_\infty}(x) = \left\{
	\begin{array}{l l}
		0 & \quad d(x,P) > 0\\
		1 & \quad d(x,P) = 0\\
	\end{array} \right.
\]

On the other hand, if $f(\e) = 2\e$, $\mu_{L_n}(x) = \Delta_1(2^{-n}d(x,P))$. As $n\rightarrow \infty$, $2^{-n} d(x,P) \rightarrow 0$, and hence $\mu_{L_\infty}(x) = 1$ $\forall x \in \Omega$. 
\end{exa}

The behaviour of the hedges given in example \ref{exa:detcom} is therefore different from those in theorems \ref{thm:dilcom} and \ref{thm:concom}, since the concept either shrinks to a single point, in the case of contraction, or, in the case of dilation, expands to fill the entire space $\Omega$.
 
\section{Discussion}
\label{sec:disc}
We have presented formulae for linguistic hedges which are both functional and semantically grounded. The modifiers presented arise naturally from the label semantics framework, in which concepts are represented by a prototype and threshold. Our hedges have an intuitive meaning: if I think that the threshold for a concept `small' is of a certain width, then the threshold for the concept `very small' will be narrower. On the other hand, the threshold for the concept `quite small' will be broader. The hedges proposed are examples of `type 1' hedges, i.e. they operate equally across all dimensions of the fuzzy set associated with a concept. The first result presented is somewhat similar to a powering modifier since the core and support of the set remain the same. In \cite{boscempty, marin2003}, it is argued that this property is undesirable for hedges used in fuzzy expert systems, since if a query is returning too large a set of answers, this type of contraction hedge does not reduce this overabundance. However, although the hedges we propose do not at their simplest address the overabundance issue, we argue that they address another problem associated with powering hedges, in that they have a clear semantic grounding that the powering modifiers lack.

\cite{bosc, decock2002} also propose modifiers that are semantically grounded, using the idea of resemblance to nearby objects. Their formulations have the properties that the core and support of the fuzzy set are both changed, thereby addressing the issue of overabundant answers \cite{boscempty, marin2003}. In our most specific case, since the prototype is not altered, the core and support of the fuzzy set representing the concept remain the same. However, our initial proposal can be generalised, as in section \ref{sub:diff}, to apply to the case where $P_1 \neq P_2$. Specifying a semantically meaningful way of altering the boundaries of the prototype would answer the objection that the core and support of a set should change under a linguistic hedge.

The most general result (section \ref{sub:f}) shows that the formula still applies when $\e_2 \leq f(\e_1)$, or $\e_2 \geq f(\e_1)$. Combined with a distribution $\delta$ such that the lower bound of the distribution is not zero, the core and support of the fuzzy set are modified. With the condition $\e_2 = f(\e_1)$, we are able to recreate the result given in \cite{novak} for linear membership functions, and show how the model proposed by \cite{bosc} has strong similarities to our own. In this case we have introduced additional parameters, so the simplicity of the original result is lost. However, the further parameters introduced are no more than those introduced by \cite{novak}, and arguably fewer than those introduced by \cite{bosc}, who require that a resemblance relation be specified, using two additional parameters, and also, that a $T$-norm or fuzzy implication need to be specified. There are various choices of operator that could be used for either of these, and it is not obvious that any one is better than the others. 

We have also shown that the basic case operators `very' and `quite' can be composed, which is not immediately obvious from the formulae (section \ref{sub:comp}). Further, we show that in the limit of composition the membership of any object in the fuzzy part of $L$, i.e. $\{x: 0 < \mu_L(x) < 1\}$, increases to 1 in the case of `quite' or decreases to 0 in the case of `very' (section \ref{sub:limits}). This is similar to the limit of applying the powering modifiers, but differs from what would happen with the modifiers proposed by \cite{bosc}. In that case, the limit of `very' would shrink to a single point and the limit of `quite' would expand to encompass the whole universe of discourse. This can be modelled using the deterministic hedges described in section \ref{sub:det}. Although behaviour differs slightly, in fact human discourse does not  apply modifiers infinitely, so the difference in behaviour is arguably not important.

Our  formulation has the benefit that it can be applied in more situations than simply linguistic hedges. For example, the concept `apple green' has a prototype different to that of just green, and the threshold for `apple green' is likely to be smaller than the threshold for simply `green'. Our model can take account of this.

\section{Conclusions and further work}
We have presented formulae for two simple linguistic hedges, `very' and `quite'. These formulae are functional, hence easy to compute, but also semantically grounded, in that they arise naturally from the conceptual framework of label semantics combined with prototype theory and conceptual spaces theory, and in the most specific case require no additional parameters. We have also shown that two other formulations \cite{bosc, decock2002, novak},  can be derived from this framework with equal or fewer parameters. We have shown that the hedges can be composed and have described their behaviour in the limit of composition.

Further work could look at testing the utility of these hedges in particular classifiers to compare their performance with the classical hedges and with the hedges used by e.g. \cite{boscempty, marin2003}, and also to examine a trade-off between accuracy and the number of parameters used. Alternatively, investigating semantically grounded ways of expanding or reducing prototypes could have a similar impact.

The model could also be extended to the more complicated type 2 hedges such as `essentially', or `technically', by treating dimensions of the conceptual space heterogeneously. This requires using some type of weighting or necessity measure on the dimensions, work which is currently ongoing.

\section{Acknowledgements}
Martha Lewis gratefully acknowledges support from EPSRC Grant No. EP/E501214/1

\bibliographystyle{plain}

\end{document}